\documentclass[10pt,journal,compsoc]{IEEEtran}

\usepackage{ifpdf}

\ifCLASSINFOpdf
   \usepackage[pdftex]{graphicx}
   \graphicspath{{../imgs/}}
   \DeclareGraphicsExtensions{.pdf,.jpeg,.png}
\else
\fi

\usepackage{graphicx}
\usepackage{adjustbox}

\usepackage{amsmath,amssymb,amsthm}
\usepackage{amscd}

\usepackage{algorithm}
\usepackage{algorithmic}

\usepackage{booktabs}

\usepackage{mdwmath}
\usepackage{mdwtab}

\usepackage{eqparbox}

\ifCLASSOPTIONcompsoc
 \usepackage[caption=false,font=footnotesize,labelfont=sf,textfont=sf]{subfig}
\else
 \usepackage[caption=false,font=footnotesize]{subfig}
\fi

\usepackage{stfloats}

\usepackage{url}

\usepackage{multirow}

\usepackage{xcolor}

\RequirePackage[colorlinks,citecolor=blue,urlcolor=blue]{hyperref}

\definecolor{dgreen}{rgb}{0.00,0.49,0.00}
\definecolor{dblue}{rgb}{0,0.08,0.75}
\hypersetup{
    colorlinks = true,
    citecolor=dgreen,
    urlcolor=dblue,
    linkcolor=dblue
}

\usepackage{thmtools, thm-restate}
\usepackage[capitalize]{cleveref}
\usepackage[square,numbers,sort]{natbib}
\usepackage{amsfonts}       %
\usepackage{nicefrac}       %

\renewcommand{\paragraph}[1]{\vspace{1em}\noindent{\bfseries #1}.}

\declaretheorem[name=Theorem,refname=Thm.]{theorem}

\declaretheorem[name=Remark]{remark}

\declaretheorem[name=Assumption,refname=Asm.]{assumption}

\crefname{assumption}{Assumption}{Assumptions}
\crefname{equation}{}{}
\Crefname{equation}{Eq.}{Equations}
\crefname{figure}{Fig.}{Figs.}
\crefname{table}{Tab.}{Tabs.}
\crefname{section}{Sec.}{Sec.}
\crefname{theorem}{Thm.}{Thm.}
\crefname{lemma}{Lemma}{Lemmas}
\crefname{corollary}{Cor.}{Cor.}
\crefname{example}{Example}{Examples}
\crefname{remark}{Remark}{Remarks}
\crefname{algorithm}{Alg.}{Algorightms}

\crefname{appendix}{Appendix}{Appendices}
\crefname{subappendix}{Appendix}{Appendices}
\crefname{subsubappendix}{Appendix}{Appendices}

\newcommand{\paperTitle}{Robust Meta-Representation Learning via Global Label Inference and Classification}

\newcommand{\fsl}{FSL}

\newcommand{\laml}{MeLa}

\newcommand{\imgnet}{\textsc{ImageNet}}
\newcommand{\mimg}{\textit{mini}\imgnet{}}
\newcommand{\timg}{\textit{tiered}\imgnet{}}

\newcommand{\queryn}{n_q}
\newcommand{\supportn}{n_s}

\newcommand{\mbf}[1]{\mathbf{#1}}

\newcommand{\R}{{\mathbb{R}}}

\newcommand{\gls}{\ensuremath{\textrm{\tiny GLS}}}
\newcommand{\EE}{\mathbb{E}}
\newcommand{\Lagr}{\mathcal{L}}
\newcommand{\lagr}{\ell}

\newcommand{\Lce}{\Lagr_{\rm ce}}
\newcommand{\lce}{\lagr_{\rm ce}}

\newcommand{\argmin}{\operatornamewithlimits{argmin}}

\newcommand{\X}{{\mathcal{X}}}
\newcommand{\Y}{{\mathcal{Y}}}

\newcommand{\F}{{\mathcal{F}}}
\newcommand{\D}{{\mathcal{D}}}

\newcommand{\E}{{\mathcal{E}}}

\newcommand{\alg}{\textrm{Alg}}

\renewcommand{\citet}{\cite}

\renewcommand{\paragraph}[1]{\noindent{\bfseries #1.}}

\newcommand{\eqals}[1]{\begin{align*}#1\end{align*}}

\newcommand{\eqal}[1]{\begin{align}#1\end{align}}

\newcommand{\metapar}{\theta}
\newcommand{\Metapar}{\Theta}
\newcommand{\dtr}{S}
\newcommand{\dval}{Q}
\newcommand{\dm}{D_{\rm global}}
\newcommand{\Tau}{\mathcal{T}}
\newcommand{\embd}{g_{\metapar}}

\newcommand{\gfsl}{GFSL}

\providecommand{\nor}[1]{\left\|{#1}\right\|}

\newcommand{\ridge}{w}

\newcommand{\unif}{\ensuremath{\textrm{Unif}}}

\usepackage[multiple]{footmisc}

\begin{document}
\title{\paperTitle}

\newcommand{\MP}[1]{\textcolor{orange}{#1}}

\author{Ruohan~Wang\thanks{Agency for Science, Technology and Research, Singapore},
        John Isak Texas Falk\thanks{Centre for Artificial Intelligence, Computer Science Department, University College London, London, United Kingdom}\footnotemark[3],
        Massimiliano Pontil\footnotemark[2]\thanks{CSML, Isitituto Italiano di Tecnologia, Genova, Italy},
        and~Carlo~Ciliberto\footnotemark[2]%
}

\markboth{IEEE Transactions on Pattern Analysis and Machine Intelligence}%
{Wang \MakeLowercase{\textit{et al.}}: \paperTitle}

\IEEEtitleabstractindextext{%
\begin{abstract}
Few-shot learning (FSL) is a central problem in meta-learning, where learners must efficiently learn from few labeled examples.
Within FSL, feature pre-training has become a popular strategy to significantly improve generalization performance.
However, the contribution of pre-training to generalization performance is often overlooked and understudied, with limited theoretical understanding.
Further, pre-training requires a consistent set of global labels shared across training tasks, which may be unavailable in practice.
In this work, we address the above issues by first showing the connection between pre-training and meta-learning.
We discuss why pre-training yields more robust meta-representation and connect the theoretical analysis to existing works and empirical results.
Secondly, we introduce Meta Label Learning (\laml{}), a novel meta-learning algorithm that learns task relations by inferring global labels across tasks.
This allows us to exploit pre-training for FSL even when global labels are unavailable or ill-defined.
Lastly, we introduce an augmented pre-training procedure that further improves the learned meta-representation.
Empirically, \laml{} outperforms existing methods across a diverse range of benchmarks, in particular under a more challenging setting where the number of training tasks is limited and labels are task-specific.
\end{abstract}

}

\maketitle
\thispagestyle{plain}

\IEEEdisplaynontitleabstractindextext

\IEEEpeerreviewmaketitle

\ifCLASSOPTIONcompsoc
\IEEEraisesectionheading{\section{Introduction}\label{sec:intro}}
\else
\section{Introduction}
\label{sec:introduction}
\fi

Deep neural networks have facilitated transformative advances
in machine learning in various areas~\citep[e.g.][]{silver2017mastering,goodfellow2014generative,he2016deep,brown2020language,krizhevsky2012imagenet,mnih2015human}. However, state-of-the-art models typically require labeled datasets of extremely large scale, which are prohibitively expensive to curate. When training data is scarce, neural networks often overfits with degraded performance. Few-shot learning (\fsl{}) aims to address this loss in performance by developing algorithms and architectures capable of learning from few labeled samples.

Meta-learning~\cite{hospedales2020meta, vanschoren2019meta} is a popular class of algorithms to tackle \fsl{}. Broadly, meta-learning seeks to learn transferable knowledge over many \fsl{} tasks, and to apply such knowledge to novel ones. For instance, Model Agnostic Meta Learning (MAML)~\cite{finn2017model} learns a prior over the model initialization that is suitable for fast adaptation. Existing meta-learning methods for tackling \fsl{} may be loosely classified into three categories; optimization~\citep[e.g.][]{finn2017model, bertinetto2018meta, wertheimer2021few}, metric learning~\citep[e.g.][]{vinyals2016matching, snell2017prototypical, sung2018learning}, and model-based methods~\citep[e.g.][]{ha2016hypernetworks, rusu2018meta, qi2018low}. The diversity of existing strategies poses a natural question: can we derive any ``meta-insights'' from them to facilitate the design of future methods?

Among the existing methods, several trends have emerged for designing robust few-shot meta-learners. Chen \textit{et al.}~\cite{chen2018closer} observed that data augmentation and deeper networks significantly improves generalization performance. The observations have since been widely adopted~\citep[e.g.][]{tian2020rethinking, bateni2020improved, lee2019meta}. {\itshape Network pre-training} has also become ubiquitous~\citep[e.g.][]{el2022lessons, wang2020structured, rodriguez2020embedding, wertheimer2021few}, and dominates state-of-the-art models. Sidestepping the task structure and episodic training of meta-learning, pre-training learns (initial) model parameters by merging all \fsl{} tasks into one ``flat'' dataset of labeled samples followed by standard multi-class classification. The model parameters may be further fine-tuned to improve performance.

Despite its popularity, the limited theoretical understanding of pre-training leads to diverging interpretations of existing methods. Many meta-learning methods consider pre-training as nothing but a standard pre-processing step, and attribute the observed performance almost exclusively to their respective algorithmic and network design~\citep[e.g.][]{ye2020few, rusu2018meta, yang2021free}. However, extensive empirical evidence suggests that pre-training is crucial for model performance~\cite{wertheimer2021few, wang2021role}. Tian \textit{et al.}~\cite{tian2020rethinking} demonstrated that simply learning task-specific linear classifiers over the pre-trained representation outperforms various meta-learning strategies. Wertheimer \textit{et al.}~\cite{wertheimer2021few} further showed that earlier \fsl{} methods also benefit from pre-training, resulting in improved performance. 

In this work we contributes a unified perspective by showing that pre-training directly relates to meta-learning by minimizing an upper bound on the meta-learning loss. In particular, we show that pre-training achieves a smaller expected error and enjoys a better convergence rate compared to its meta-learning counterpart. More broadly, we connect pre-training to conditional meta-learning~\citep{denevi2020advantage, wang2020structured}, which has favorable theoretical properties including tighter bounds. Our result provide a principled justification of why pre-training yields a robust meta-representation for \fsl{}, and the associated performance improvement.

Motivated by this result, we propose an augmentation procedure for pre-training that quadruples the number of training classes by considering rotations as novel classes and classifying them jointly. This significantly increases the size of training data and leads to robust representations. We empirically demonstrate that the augmentation procedure consistently performs better across different benchmarks.

\begin{figure*}[t]
    \centering
    \subfloat[\textbf{Global vs. Local labels}]{
    \includegraphics[width=0.2\textwidth]{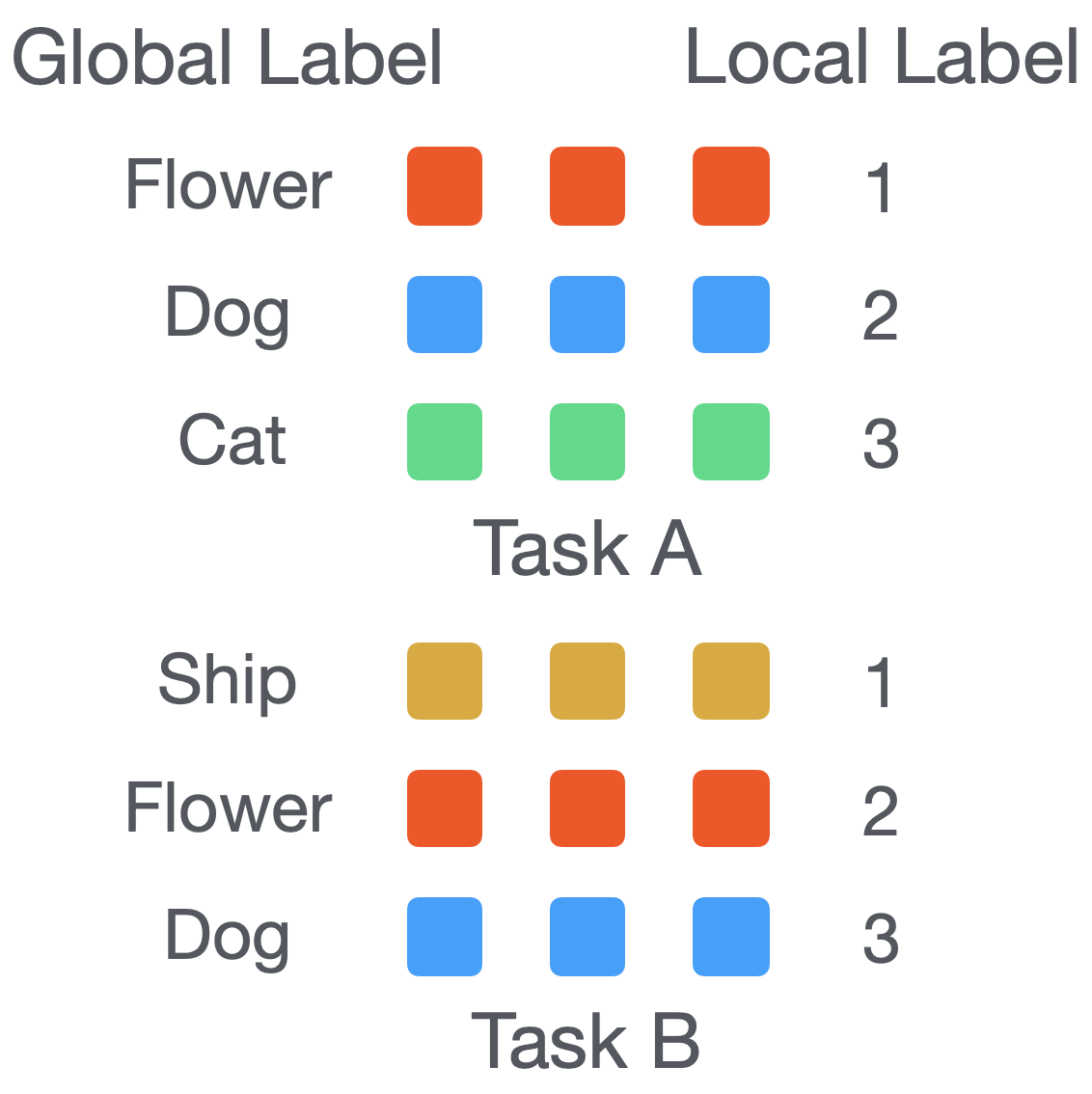}\label{fig:global_local_label}}
    \hfil
    \subfloat[\textbf{Global to Local Classifier}]{
    \includegraphics[width=0.6\textwidth]{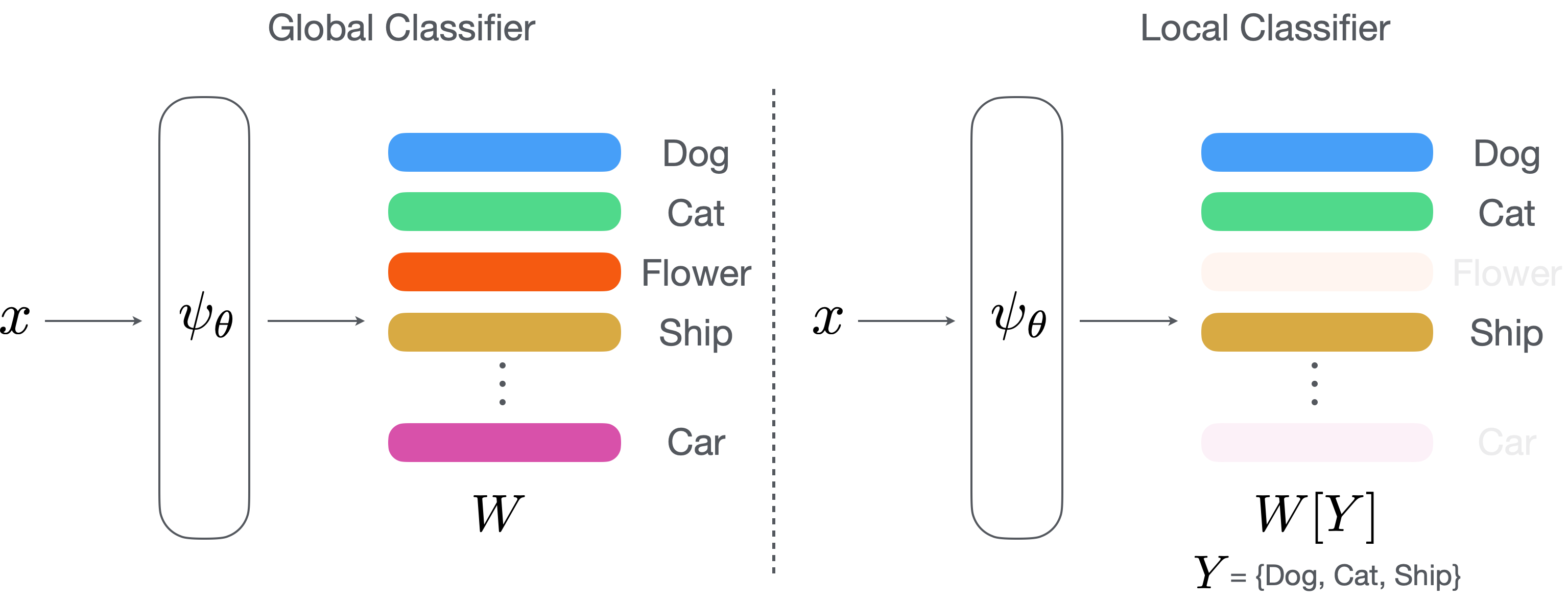}\label{fig:global_local_model}}
    \caption{{\bfseries (a)} Colored squares represent samples. Tasks A and B can be ``merged'' meaningfully using global labels, but not local ones. {\bfseries (b)} A global classifier can be used as local classifiers given the indices $Y$ of the intended classes to predict.}
\end{figure*}

The standard \fsl{} setting~\citep[e.g.][]{finn2017model, nichol2018first, bertinetto2018meta} assumes access to a collection of tasks (i.e. the meta-training set) for training data. To perform pre-training, meta-training tasks must be merged into a flat dataset (see \cref{sec:bg_pre-train} for a formal definition), which implicitly assumes access to some notion of {\it global labels} shared across all tasks. However, global labels may be unavailable, such as when each task is independently labeled with only \textit{local labels}. This renders naive task merging and pre-training infeasible (see \cref{fig:global_local_label}). Independent task annotation is a more realistic and general assumption, capturing scenarios when training tasks are collected \textit{organically} from different sources rather than generated {\itshape synthetically} from a base dataset. Practical scenarios where naive task merging is infeasible include non-descriptive task labels (e.g. numerical ones) or concept overlap (e.g. marine animals vs mammals) among labels.

To tackle independent task annotation, we propose {\bfseries Me}ta {\bfseries La}bel Learning (\laml{}), a novel algorithm that automatically infers a notion of \textit{latent} global labels consistent with local task constraints. The inferred labels enable us to exploit pre-training for \fsl{}, and to bridge the gap between experimental settings with or without access to global labels. Empirically, we demonstrate that \laml{} is competitive with pre-training on oracle labels.

For experiments, we introduce a new Generalized \fsl{} (\gfsl{}) setting. In addition to independent task annotation, we also adopt a fixed-size meta-training set and enforce no repetition of samples across tasks. This challenging setting evaluates how efficiently meta-learning algorithms generalize from limited number of tasks, and prevents the algorithms from trivially uncover task relations by implicitly matching identical samples across tasks. We empirically show that \laml{} performs robustly in both standard and \gfsl{} settings, and clearly outperforms state-of-the-art models in the latter.\\

\noindent We summarize the main contributions below:
\begin{itemize}
    \item We prove that pre-training relates to meta-learning as a loss upper bound. Consequently, minimizing the pre-training loss is a viable proxy for tackling meta-learning problems. Additionally, we identify meta-learning regimes where pre-training offers a clear improvement with respect to sample complexity. This theoretical analysis provides a principled explanation for pre-training's empirical advantage. 
    \item We propose \laml{}, a general algorithm for inferring latent global labels from meta-training tasks. It allows us to exploit pre-training when global labels are absent or ill-defined.
    \item We propose an augmented pre-training procedure for \fsl{} and a \gfsl{} experimental setting.
    \item Extensive experiments demonstrate the robustness of \laml{}. Detailed ablations provide deeper understanding of the model.
\end{itemize}

\paragraph{Extension of \cite{wang2021role}} This paper is an extended version of \cite{wang2021role} with the following contributions in addition to those of the original work: $i)$ a deeper theoretical insight into the role of pre-training from the perspective of the risk (rather than the empirical risk as in \cite{wang2021role}), and quantifying its benefit in terms of sample complexity, $ii)$ the augmented training procedure for \fsl{}, $iii)$ the \gfsl{} experimental setting, $iv)$ significantly more empirical evidence to support the proposed algorithm. 

\section{Background}
\label{sec:bg}
We formalize \fsl{} as a meta-learning problem and review related methods. We also discuss the pre-training procedure adopted by many \fsl{} methods.

\subsection{Few-shot Learning using Meta-learning}
\fsl{}~\citep{fei2006one} considers a meta-training set of tasks $\Tau=\{(\dtr_t, \dval_t)\}_{t=1}^{T}$, with {\em support set} $\dtr_t=\{(x_{j}, y_{j})\}_{j=1}^{\supportn}$ and {\em query set} $\dval_t=\{(x_j, y_j)\}_{j=1}^{\queryn}$ sampled from the same distribution. Typically, $\dtr_t$ and $\dval_t$ contain a small number of samples $\supportn$ and $\queryn$ respectively. We denote by $\D$ the space of datasets of the form $S_t$ or $Q_t$.

The meta-learning formulation for \fsl{} aims to find the best {\itshape base learner} $\alg(\metapar,\cdot):\D\to\mathcal{F}$ that takes as input support sets $\dtr$, and outputs predictors $f = \alg(\metapar,\dtr)$, such that predictions $y=f(x)$ generalize well on the corresponding query sets $\dval$. The base learner is meta-parameterized by $\metapar\in\Metapar$. Formally, the meta-learning objective for \fsl{} is
\eqal{\label{eq:meta-learning-risk}
\EE_{(\dtr, \dval)\in\Tau} ~\Lagr\big(\alg(\metapar,\dtr),~\dval\big),
}
where $\EE_{(\dtr, \dval)\in\Tau}\triangleq \frac{1}{|\Tau|}\sum_{(S,Q)\in\Tau}$ is the empirical average over the meta-training set $\Tau$. The {\itshape task loss} $\Lagr:\F\times\D\to\R$ is the empirical risk of the learner $f$ over query sets, based on some {\itshape inner loss} $\ell: \Y\times\Y\to\R$, where $\Y$ is the space of labels,
\eqal{\label{eq:meta-loss}
    \Lagr(f,D) ~=~ \EE_{(x, y)\in D}~[ \ell(f(x),y)].
}
\Cref{eq:meta-learning-risk} is sufficiently general to describe most existing methods. For instance, model-agnostic meta-learning (MAML)~\cite{finn2017model} parameterizes a model $f_\metapar:\X\to\Y$ as a neural network, and $\alg(\metapar, D)$ performs one (or more) steps of gradient descent minimizing the empirical risk of $f_\metapar$ on $D$. Formally, given a step-size $\eta>0$, 
\eqal{\label{eq:maml}
f_{\metapar'} = \alg(\metapar,D) \quad ~\textrm{with}~ \quad \metapar' = \metapar - \eta~ \nabla_\metapar \Lagr(f_\metapar,D).
} 
Clearly, base learners $\alg(\metapar,\cdot)$ is key to model performance and various strategies have been explored. Our proposed method is most closely related to meta-representation learning~\citep{raghu2019rapid, lee2019meta, bertinetto2018meta,franceschi2018bilevel}, which parameterizes the base learner as $A(\theta, D) = \ridge(\embd(D)) \embd(\cdot)$, consisting of a global feature extractor $\embd:\X\to\R^m$ and a task-adaptive classifier $\ridge:\D\to\{f: \R^{m} \to \Y\}$ that optimizes the following
\eqal{\label{eq:meta-representation-model}
   \min_{\metapar\in\Metapar} ~~ \EE_{(\dtr, \dval)\in\Tau} \left[\mathcal{L}\big(\ridge(\embd(\dtr)), \embd(\dval)\big)\right]
}
where $\embd(D)\triangleq \{(\embd(x), y) ~|~ (x, y) \in D\}$ is the embedded dataset. \Cref{eq:meta-representation-model} specializes \cref{eq:meta-learning-risk} by learning a feature extractor $\embd$ shared (and fixed) among tasks. Only the classifier returned by $\ridge(\cdot)$ adapts to the current task, in contrast to having the entire model $f_\metapar:\X\to\Y$ adapted (e.g. \cref{eq:maml} for MAML). While this may appear to restrict model adaptability, \cite{raghu2019rapid} has demonstrated that meta-representation learning matches MAML's performance. Moreover, they showed that feature reuse is the dominant contributor to the generalization performance rather than adapting the representation to the task at hand.

The task-adaptive classifier $\ridge(\cdot)$ may take various forms, including nearest neighbor~\cite{snell2017prototypical}, ridge regression classifier~\cite{bertinetto2018meta}, embedding adaptation with transformer models~\cite{ye2020few}, and Wasserstein distance metric~\cite{zhang2020deepemd}. In particular, the ridge regression estimator
\eqal{\label{eq:closed-form-solver}
    \ridge_{\rm ridge}(D) {=} \argmin_{W}~\EE_{(x,y)\in D}~\nor{Wx - y}^2 {+} \lambda_1\nor{W}_{F}^2,
}
where \(\nor{\cdot}_{F}\) is the Frobenius norm, admits a differentiable closed-form solution and is computationally efficient for optimizing \cref{eq:meta-representation-model}.

\subsection{Conditional Meta-Learning}
Conditional formulations of meta-learning~\cite{denevi2020advantage,wang2020structured} extends \Cref{eq:meta-learning-risk} by considering base learners of the form $\alg(\tau(Z),\dtr)$, where the meta-parameters $\metapar = \tau(Z)$ is conditioned on some ``contextual'' information $Z \in \mathcal{Z}$ about the task $\dtr$. Assuming each task in the meta-training set $\Tau$ to be equipped with such contextual information, \cref{eq:meta-learning-risk} can be re-expressed as
\eqal{\label{eq:conditional-meta-learning-risk}
\min_{\tau:\mathcal{Z}\to\Metapar}~~ \EE_{(\dtr, \dval, Z)\in\Tau} ~\Lagr\big(\alg(\tau(Z),\dtr),~\dval\big),
}
namely the problem of learning a function $\tau:\mathcal{Z}\to\Metapar$, which maps contextual information $Z\in\mathcal{Z}$ (e.g. a textual meta-description of the task/dataset) to a good task-specific base learner with parameters $\metapar = \tau(Z)$.

The conditional formulation seeks to capture complex (e.g. multi-modal) distributions of meta-training tasks, and uses a unique base learner tailored to each one. In particular, \cite{vuorio2019multimodal, yao2019hierarchically, rusu2018meta} directly learn data-driven mappings from target tasks to meta-parameters, and \cite{jiang2018learning} conditionally transforms feature representations based on a metric space trained to capture inter-class dependencies. Alternatively, \cite{jerfel2019reconciling} considers a mixture of hierarchical Bayesian models over the parameters of meta-learning models in order to condition on target tasks. In \cite{wang2020structured}, Wang \textit{et al.} showed that conditional meta-learning can be interpreted as a structured prediction problem and proposed a method leveraging structured prediction. From a more theoretical perspective, Denevi \textit{et al.} \cite{denevi2020advantage,denevi2022} proved that conditional meta-learning is theoretically advantageous compared to unconditional approaches by incurring smaller excess risk and being less prone to negative transfer. As we will discuss in \cref{sec:thm}, conditional meta-learning is closely related to our theoretical analysis on feature pre-training.

\subsection{Feature Pre-training}
\label{sec:bg_pre-train}
Feature pre-training has been widely adopted in the recent meta-learning literature~\citep[e.g.][]{mangla2020charting, oreshkin2018tadam, chen2018closer, wang2020structured, wertheimer2021few, yang2021free, ye2020few, bateni2020improved, requeima2019fast, sun2019meta, zhang2020deepemd, afrasiyabi2022matching}, and is arguably one of the key contributors to performance of state-of-the-art models. Instead of directly learning the feature extractor $\embd$ by optimizing \cref{eq:meta-representation-model}, pre-training first learns a feature extractor via standard supervised learning.

Formally, the meta-training set $\Tau$ is ``flattened'' into $\dm$ by merging all tasks:
\eqal{\label{eq:merge_set}
    \dm = D(\Tau) =\{(x_i, y_i)\}_{i=1}^N=\bigcup\limits_{(\dtr, \dval) \in \Tau} (\dtr\cup\dval),
}
where we have re-indexed the $(x_i,y_i)$ samples from $i=1$ to $N$ (the cumulative number of points from all support and query sets) to keep the notation uncluttered. Pre-training then learns the embedding function $\embd$ on $\dm$ using the standard cross-entropy loss $\ell_{\rm ce}$ for multi-class classification:
\eqal{
\label{eq:std_classify}
    (W_N^{\rm pre},\metapar_N^{\rm pre}) = \argmin_{\metapar, W}~\Lagr(W\embd,\dm).
}
where $W$ is the linear classifier over all classes. After pre-training, the feature extractor is either fixed~\citep[e.g.][]{rusu2018meta, ye2020few, wang2020structured, zhang2020deepemd, tian2020rethinking} or further adapted~\citep[e.g][]{bateni2020improved, bateni2022beyond, goldblum2020unraveling, requeima2019fast} via meta-learning.

There is limited theoretical understanding and consensus on the effect of pre-training in \fsl{}. In \cite{rusu2018meta, bateni2020improved, yang2021free}, the pre-training is only considered a standard pre-processing step for encoding the raw input and model performance is predominantly attributed to the proposed meta-learning algorithms. In \cite{goldblum2020unraveling} the authors similarly argued that meta-trained features are better than pre-trained ones, observing that adapting the pre-trained features with several base learners resulted in worse performance compared to the meta-learned features. In contrast, however, several works also empirically demonstrated that pre-training contributes significantly towards performance. \cite{tian2020rethinking} showed that combining the pre-trained features with suitable base learners already outperforms various meta-learning methods, while \cite{el2022lessons} observed that pre-training dominates top entries in 2021 Meta-learning Challenge.

We conclude by noting that recently pre-training has also been successfully applied to large-scale multi-modal settings combining visual and language input, enabling zero-shot learning~\cite{radford2021learning}, more flexible few-shot learning~\cite{alayrac2022flamingo} (i.e., tasks may be described using free text), and generating more samples to augment \fsl{}~\cite{xu2022generating}. While this line of work further showcases the efficacy of pre-training strategies for \fsl{}, in this work we focus on few shot learning settings with a single input modality.

\section{Pre-training as Meta-learning}
\label{sec:thm}
In this section, we characterize how feature pre-training relates to meta-learning as a loss upper bound. More precisely, we show that pre-training induces a special base learner with its corresponding meta-learning loss upper bounded by the cross-entropy loss $\lce$. Consequently, pre-training already produces a meta-representation suitable for \fsl{}, matching the empirical results from \cite{tian2020rethinking, ye2020few}. In addition, we show that pre-training incurs a smaller risk compared to its meta-learning counterpart, and more generally induces a conditional formulation that exploits contextual information for more robust learning.

\subsection{Notation and Problem Setting}
We consider a few-shot classification setting with a total of $C$ classes (global labels). Denote by $\mu$ the meta-distribution sampling distributions (a.k.a. tasks) $\rho$, from which we sample support and query sets $(S,Q)$ for each task. Each task distribution $\rho$ is associated with $k\leq C$ class labels $y_\rho^{(1)},\dots,y_\rho^{(k)}\in\{1,\dots,C\}$. Denote by $\rho_Y = \{y_\rho^{(1)},\dots,y_\rho^{(k)}\}$ the corresponding subset of $\{1,\dots,C\}$. Given a matrix $W\in\R^{C\times m}$ and a vector $Y\in\{1,\dots,C\}^k$ of indices, we denote by $W[Y] = W[\rho_Y]\in\R^{k\times m}$ the submatrix of $W$ obtained by selecting the rows corresponding to the unique indices $\rho_Y$ in $Y$. Lastly, Given a dataset $D = (x_i,y_i)_{i=1}^n$ we denote by $D_Y\in\{1,\dots,C\}^n$ the vector with entries corresponding to the labels $y_i$.

We also define the expected error incurred by a meta-learning algorithm solving \cref{eq:meta-representation-model}
\eqal{\label{eq:gls-risk}
    \E(\theta,\mu) = \EE_{\rho\sim\mu}~\EE_{(S,Q)\sim\rho}~\Lagr(w(g_\theta(S)),g_\theta(Q)).
}
This is the meta-learning risk incurred by a meta-parameter $\theta$, namely the error incurred by training the classifier via $w(g_{\theta}(S))$ (e.g. \Cref{eq:closed-form-solver}) and testing it on the query set $g_{\theta}(Q)$, averaged over $(S, Q)$ pairs sampled from tasks $\rho$, which in turn are sampled from meta-distribution $\mu$. The risk is the ideal error we wish to minimize.

\subsection{Global Label Selection (GLS)}\label{sec:gls}
We start our analysis by introducing a special \fsl{} scenario that will be useful for understanding the relation between pre-training and meta-learning. To this end, we assume in this scenario that global labels are available to the model.

Given the access to global labels, we can design a new algorithm that learns a single global multi-class linear classifier $W$ at the meta-level (i.e. shared across all tasks), and simply selects the required rows $W[S_Y]$ when tackling a task. More formally, we can define a special base learner called {\itshape global label selector (GLS)} such that
\eqals{
    \alg((W,\metapar),S) = \textrm{GLS}(W,\metapar,S) = W[S_Y]\embd(\cdot).
}
Illustrated in \cref{fig:global_local_model}, this ``algorithm'' does not solve an optimization problem on the support set $S$, but only selects the subset of rows of $W$ corresponding to the classes present in $S$ as the task-specific classifier.

Since $W$ and $\theta$ are now both shared across all tasks, we may learn them jointly by minimizing the following 
\eqal{\label{eq:erm-gls}
    \EE_{(S,Q)\in\Tau}~\Lagr(W[S_Y]\embd(\cdot), Q),
}
over both $W$ and $\theta$. This strategy, to which we refer as {\itshape meta-GLS}, learns both the representation and linear classifier at the meta-level, with the sole task-specific adaptation process being the selection of columns of $W$ using the global labels.

\vspace{1em}\paragraph{GLS Finds a Good Meta-representation}
Learning a global $W$ shared among multiple tasks (rather than having each classifier $w(\embd(S))$ accessing exclusively the tasks' training data), can be very advantageous for generalization. This is evident when  the (global) classes are separable for a meta-representation $g_{\metapar}$. Let
\eqal{\label{eq:expected-gls}
    \E_\gls(W,\metapar,\mu) = \EE_{\rho\sim\mu}~\EE_{(S,Q)\sim\rho}~\Lagr(W[S_Y]\embd(\cdot), Q),
}
denote the expected meta-GLS risk incurred by the minimizer of \cref{eq:erm-gls}. Then, for any inner algorithm $w(\cdot)$, we have
\eqal{\label{eq:gls-and-metarep}
    \min_{W} ~\E_\gls(W,\metapar,\mu) \leq \E(\metapar,\mu),
}
namely that, for any given representation $g_\theta$, finding a global classifier $W$ for all classes is more favorable than solving each task in isolation. In other words, {\itshape
solving meta-GLS provides a good representation $\embd(\cdot)$ for standard meta-learning problem. 
} 
\subsection{Pre-training and GLS}
Existing works such as \cite{tian2020rethinking, el2022lessons} demonstrates that pre-training offers a robust alternative to learn the meta-representation $\embd(\cdot)$. We will show that GLS is related to pre-training, under some mild assumptions.

\begin{assumption}\label{asm:task-distribution}
The meta-distribution $\mu$ samples tasks $\rho$. Sampling from each $\rho$ is performed as follows:
\begin{enumerate}
    \item For each $j\,{\in}\,\{1,\dots,k\}$ and class $y_\rho^{(j)} \,{\in}\, \rho_Y$, 
    we sample $n$ examples $x_{1}^{(j)},\dots,x_n^{(j)}$ i.i.d. from a conditional distribution $\pi(x|y\,{=}\,y_\rho^{(j)})$ shared across all tasks. All generated pairs are collected in the support set $S = (x_i^{(j)},y_\rho^{(j)})_{i,j=1}^{n,k}$.
    \item The query set $Q$ is generated by sampling  $m$ points i.i.d. from $\pi_\rho(x,y)$, namely $Q\sim\pi_\rho^m$ with
    \eqal{\label{eq:sampling-query-set}
        \pi_\rho(x,y) = \pi(x|y)\mathrm{Unif}_{\rho}(y)
    }
    and\,$\mathrm{Unif}_{\rho}$\,the uniform distribution over the labels in $\rho_Y$.
\end{enumerate}
\end{assumption}
\noindent In essence, \cref{asm:task-distribution} characterizes the standard process of constructing meta-training tasks for \fsl{}, typically adopted to build pre-training datasets in practice. In particular, let $\pi_\mu(x,y)$ be the marginal probability of observing $(x, y)$ in the meta-training tasks, i.e. firstly sampling a task $\rho$ from $\mu$, followed by sampling a class $y$ uniformly by $\mathrm{Unif}_\rho(\cdot)$ and finally $x$ by $\pi(\cdot|y)$. It then follows that sampling a dataset $\dm$ from $\pi_\mu$ is equivalent to sample a meta-training set $\Tau$ from $\mu$ and flatten it into $D(\Tau)$ according to the pre-training procedure described in \cref{eq:merge_set}.

We can therefore introduce the {\itshape (global) multi-class classification} risk associated to $\pi_\mu$
\eqal{\label{eq:risk-multiclass}
\Lagr(W\embd(\cdot),\pi_\mu) = \EE_{(x,y)\sim\pi_\mu}~\ell_{\rm ce}(W\embd(x),y).
}
The above risk can be seen as the ideal objective for the pre-training estimator in \cref{eq:std_classify}. In addition, the following result relates pre-training to meta-GLS.

\begin{restatable}{theorem}{glsandpretraining}
\label{thm:gls-and-pretraining}
Under \cref{asm:task-distribution}, let $\pi_\mu(x,y)$ be the marginal distribution of observing $(x, y)$ in the meta-training set. Then, for any (global) classifier $W$, 
\eqal{\label{eq:gls-and-pretraining}
    \E_\gls(W,\metapar,\mu) \leq \Lagr(W,\theta,\pi_\mu).
}
Moreover, if the global classes are separable, 
\eqal{\label{eq:gls-and-pretraining-min}
    \min_{W,\metapar}~\E_\gls(W,\metapar,\mu) = \min_{W,\metapar}~ \Lagr(W,\theta,\pi_\mu).}
\end{restatable}
\noindent The result shows that the GLS error is upper bounded by the global multi-class classification error. Hence, minimizing the global multi-class classification error also indirectly minimizes the meta-learning risk. This implies that {\itshape pre-training implicitly learns a meta-representation suitable for \fsl{}}.

\subsection{Generalization Properties}\label{sec:generalization}
\cref{thm:gls-and-pretraining} shows that under the class-separability assumption, pre-training is equivalent to performing meta-GLS. We now study which of the two approaches is more sample-efficient from a generalization perspective.

Let $(W_T^\gls,\metapar_T^\gls)$ denote the meta-parameters learned by an algorithm minimizing \cref{eq:erm-gls} over a dataset $\Tau$ comprising of $T$ separate tasks. Applying standard results from statistical learning theory, we can obtain excess risk bounds characterizing the quality of $\theta_T$'s predictions in terms of the number $T$ of tasks the algorithm has observed in training. For instance, following \citep[Chapter 26]{shalev2014understanding} we have that in expectation with respect to sampling $\Tau$
\eqals{
    \EE_{\mathcal{T}}[\E_{\textrm{\tiny GLS}}(W_T^{\textrm{\tiny GLS}},\metapar_T^{\textrm{\tiny GLS}},\mu)] & \leq \min_{(W,\metapar)\in\Omega}
~\E_{\textrm{\tiny GLS}}(W,\metapar,\mu) + 2L_\gls\mathfrak{R}_T(\Omega)\\
 & \leq \min_{(W,\metapar)\in\Omega}
~\E_{\textrm{\tiny GLS}}(W,\metapar,\mu) + \frac{2L_\gls C_{\Omega}}{\sqrt{T}}
}
where $L_\gls$ denotes the Lipschitz constant of $\E_\gls$, while $\Omega\subset\R^{m \times C}\times\Theta$ is the space of hypotheses for the multi-class classifier $Wg_\theta(\cdot)$. Here, $\mathfrak{R}_T(\Omega)$ is the Rademacher complexity of $\Omega$ \cite{shalev2014understanding}, which measures the overall potential expressivity of an estimator that can be trained over them. For neural networks, \cite{golowich2018size} showed that $\mathfrak{R}_T(\Omega)$ may be further bounded by $\mathfrak{R}_T(\Omega) \leq C_{\Omega}/\sqrt{T}$, where $C_\Omega$ is a constant depending on the specific neural architecture, with deeper networks having a larger constant. The bound indicates that the risk incurred by GLS becomes closer to that of the ideal meta-parameters as the number of observed tasks $T$ grows.

We can apply the same Rademacher-based bounds to \cref{eq:risk-multiclass} and the pre-training estimator from \cref{eq:std_classify}, obtaining that in expectation with respect to sampling $\Tau$
\eqals{
    \EE_{\mathcal{T}}[\Lagr(W_{N}^{\rm pre},\metapar_{N}^{\rm pre},\pi_\mu)] \leq \min_{(W,\metapar)\in\Omega} ~\Lagr(W,\metapar,\pi_\mu) + \frac{2L_{\textrm{\rm pre}}C_{\Omega}}{\sqrt{N}}
}
where $N$ is the number of samples in $\dm$ and $L_{\textrm{\rm pre}}$ is the Lipschitz constant of the global multi-class classification risk. By combining the above bound with the result from \cref{thm:gls-and-pretraining} we conclude that
\eqals{
\EE_{\mathcal{T}}[~\E_{\textrm{\tiny GLS}}(W_{N}^{\rm pre},\metapar_{N}^{\rm pre},\mu)] \leq \min_{W,\metapar}
~\E_{\textrm{\tiny GLS}}(W,\metapar,\mu) + \frac{2L_{\textrm{\rm pre}}C_{\Omega}}{\sqrt{N}},
}
which is an excess risk bound analogous to that obtained for meta-GLS. The key difference is that the bound above depends on the number $N$ of total samples in $D_{\textrm{global}}$, rather that the total number $T$ of tasks. 

Comparing the rates of meta-GLS and the pre-training estimator, we observe that typically $N \gg T$ (for instance $N=nT$ when each task has the same number of $n$ samples). Additionally, since $L_{\mathrm{pre}}$ is comparable or smaller than $L_\gls$ (see Appendix \ref{sec:app-lipschitz}), we conclude that 
\begin{center}
{\itshape Given exactly the same data ($\Tau$ for meta-GLS and $D(\Tau)$ for pre-training), pre-training achieves a much smaller error than meta-GLS}.
\end{center}
For instance, in the case of a 5-way-5-shot \fsl{} problem, pre-training improves upon the meta-GLS bound on excess risk by a factor of $\sqrt{N/T} = \sqrt{n} = \sqrt{100} = 10$. 

Given the relation between GLS and standard meta-learning that we highlighted in \cref{sec:gls}, our analysis provides a strong theoretical argument in favor of adopting pre-training in meta-learning settings. To our knowledge, this is a novel and surprising result.

\subsection{Connection to Conditional Meta-Learning} More generally, we observe that GLS is also an instance of conditional meta-learning: the global labels of the task provide additional contextual information about the task to facilitate model learning. Global labels directly reveal how tasks relate to one another and in particular if any classes to be learned are shared across tasks. GLS thus simply map global labels of tasks to task classifiers via $W[S_Y]$. In contrast, unconditional approaches (e.g. R2D2~\cite{bertinetto2018meta}, ProtoNet~\cite{nichol2018first}) learn classifiers by minimizing some loss over support sets, losing out on the access to the contextual information provided by global labels.

In addition to our result, \cite{denevi2020advantage,denevi2022} also proved that conditional meta-learning is advantageous over the unconditional formulation by incurring a smaller excess risk, especially when the meta-distribution of tasks is organized into distant clusters. We refer readers to the original papers for a detailed discussion. In practice, global labels provide clustering of task samples for free and improve regularization by enforcing each cluster (denoted by global label $y_{\rho}^j$) to share classifier parameters $W[y_{\rho}^j]$ across all tasks. This provides further explanation to why pre-training yields a robust meta-representation with strong generalization performance.

\subsection{Leveraging Pre-training in Practice}\label{sec:practical-considerations}

The goal of meta-learning is to generalize to novel classes unseen during training. Therefore, practical \fsl{} applications assume meta-testing and meta-training distributions to share no class labels. To apply our analysis in \cref{sec:generalization} to these settings, we may follow the theoretical approach in \cite{du2020few} and assume that meta-training and meta-testing classes share a common representation. The assumption is reasonable since extensive empirical evidences demonstrate that pre-trained representation on meta-training set is robust for directly classifying novel classes~\citep{tian2020rethinking, el2022lessons}. To prevent overfitting on meta-training set and ensure a robust represntation for meta-testing, well-established techniques \citep[e.g.,][]{amos2017optnet, wang2020structured, tian2020rethinking} include imposing $\ell_2$ regularization during pre-training (see weight decay in Appendix \ref{sec:app_model}) and early stopping by performing meta validation.

While pre-training might offer a powerful initial representation $\metapar$ -- as highlighted by our analysis in \cref{sec:generalization} -- it may be advisable to further improve $\metapar$. One general strategy is to fine-tune $\metapar$ by directly optimizing \cref{eq:meta-representation-model} using the desired classifier to tackle novel classes~\citep[e.g.][]{ye2020few, rodriguez2020embedding,zhang2020deepemd,triantafillou2019meta}. This strategy is known as {\itshape meta fine-tuning}. A different approach is based on a transfer learning perspective. Specifically, \cite{kolesnikov2020big,shysheya2022fit,li2022cross} showed that careful task-specific fine-tuning (e.g., limiting the number of learnable parameters) from a pre-trained representation achieves robust generalization performance, even in \fsl{} settings. We investigate both strategies in our experiments.

\section{Methods}
In this section, we propose three practical algorithms motivated by our theoretical analysis. In \cref{sec:method_rot}, we introduce an augmentation procedure for pre-training to further improve representation learning in image-based tasks. In \cref{sec:method_mela}, we tackle the scenario where global labels are absent by automatically inferring a notion of global labels. Lastly, we introduce a meta fine-tuning procedure in \cref{sec:method_finetune} to investigate how much meta-learning could improve the pre-trained representation.

\subsection{Augmented Pre-training for Image-based Tasks}
\label{sec:method_rot}
In general, pre-training is a standard process with well-studied techniques for improving the final learned representation. Many of these techniques, including data augmentation~\cite{chen2018closer}, auxiliary losses~\cite{mangla2020charting} and model distillation~\cite{tian2020rethinking}, are also effective for \fsl{} (i.e. the learned representation is suitable for novel classes during meta-testing). In particular, we may interpret data augmentation techniques as increasing $N$ in the bounds for the pre-training estimator outlined in \cref{sec:generalization}, thus improving the error incurred by pre-training and consequently the learned representation $\embd$.

In addition to standard augmentations (e.g. random cropping and color jittering) investigated in \cite{chen2018closer}, we further propose an augmented procedure for pre-training via image rotation. For every class $y_i$ in the original dataset, we create three additional classes by rotating all images of class $y_i$ by $r\in\{90^\circ, 180^\circ, 270^\circ\}$ respectively. All rotations are multiples of $90^\circ$ such that they can be implemented by basic operations efficiently (e.g. flip and transpose) and prevent pre-training from learning any trivial features from visual artifacts produced by arbitrary rotations~\cite{gidaris2018unsupervised}. Pre-training is then performed normally on the augmented dataset. 

The augmented dataset quadruples the number of samples and classes compared to the original dataset. According to our analysis from \cref{sec:generalization}, pre-training on the augmented dataset may yield a more robust representation. Further, we also hypothesize that the quality of the representation also depends on the number of classes available in the pre-training dataset, since classifying more classes requires learning increasingly discriminating representations. Our experiments show that 1) augmented pre-training consistently outperforms the standard one, and 2) quality of the learned representation depends on both the dataset size and the number of classes available for training.

\subsection{Meta Label Learning}
\label{sec:method_mela}
The ability to exploit pre-training crucially depends on access to global labels. However, as discussed in \cref{sec:intro}, {\itshape global labels may be inaccessible in practical applications}. For instance when meta-training tasks are collected and annotated independently. Additionally, tasks may have conflicting labels over similar data points based on different task requirements -- a setting illustrated by our experiments in \cref{sec:exp_general}. Therefore in some applications, global labels are ill-defined, and pre-training is not directly applicable. 

To tackle this problem, we consider the more general setting where only local labels from each task are known. This setting is also the one originally adopted by most earlier works in  meta-learning~\citep[e,g][]{finn2017model, snell2017prototypical, vinyals2016matching, li2017meta, bertinetto2018meta, lee2019meta}. In the local label setting, we propose Meta Label Learning (\laml{}), which automatically infer a notion of latent global labels across tasks. The inferred labels enable pre-training and thus bridge the gap between the experiment settings with and without global labels. We stress that our proposed method {\itshape does not} replace standard pre-training with global labels, but rather provides a way to still benefit from such a strategy when they are absent.

\begin{algorithm}[t]
\caption{\laml{}}\label{alg:meta_learner}
\begin{algorithmic}
    \STATE \hspace{-1em}{\bfseries Input:} meta-training set $\Tau=\{\dtr_t, \dval_t\}_{t=1}^{T}$
    
    \STATE $\embd^{\rm sim} = \argmin_{\embd} \EE_{(\dtr, \dval)\in\Tau} \left[\mathcal{L}(\ridge(\embd(\dtr)), \embd(\dval)\right])$
    
    \STATE Global clusters $G = \textrm{LearnLabeler}(\Tau,\embd^{\textrm{sim}})$
    \STATE $\embd^{\rm pre} = \textrm{Pretrain}(D(\Tau), G)$
    
    \STATE $\embd^* = \textrm{MetaFinetune}(G, \Tau, \embd^{\rm pre})$
    
    \STATE {\bfseries Return} $\embd^*$
\end{algorithmic}
\end{algorithm}

\cref{alg:meta_learner} outlines the general approach for learning a few-shot model using \laml{}: we first meta-learn an initial representation $\embd^{\rm sim}$; Secondly, we cluster all task samples using $\embd^{\rm sim}$ as a feature map while enforcing local task constraints. The learned clusters are returned as inferred global labels. Using the inferred labels, we can apply pre-training to obtain $\embd^{\rm pre}$, which may be further fine-tuned to derive the final few-shot model $\embd^{*}$. We present in \cref{sec:method_finetune} a simple yet effective meta fine-tuning procedure.

For learning $\embd^{\rm sim}$, we directly optimize \cref{eq:meta-representation-model} using ridge regression \cref{eq:closed-form-solver} as the base learner. We use ridge regression for its computational efficiency and good performance. Using $\embd^{\rm sim}$ as a base for a similarity measure, the labeling algorithm takes as input the meta-training set and outputs a set of clusters as global labels. The algorithm consists of a clustering routine for sample assignment and centroid updates and a pruning routine for merging small clusters.\\

\paragraph{Clustering}
\label{para:clustering}
The clustering routine leverages local labels for assigning task samples to appropriate global clusters and enforcing task constraints. We observe that for any task, the local labels describe two constraints: 1) samples sharing a local label must be assigned to the same global cluster, while 2) samples with different local labels must not share the same global cluster. To meet constraint 1, we assign all samples $\{x^{(j)}_i\}_{i=1}^n$ of class $y_\rho^{(j)}$ to a single global cluster by
\eqal{
\label{eq:cluster_centroid}
    v^* = \argmin_{v \in \{1, \dots, V\}} \nor{\frac{1}{n}\sum_{i=1}^n\embd^{\rm sim}(x_i^{(j)}) - g_v}^2,
}
with $V$ being the current number of centroids.

We apply \cref{eq:cluster_centroid} to all classes $y_\rho^{(1)},\dots,y_\rho^{(k)}$ within a task. If multiple local classes map to the same global label, we simply discard the task to meet constraint 2. Otherwise, we proceed to update the centroid $g_{v^*}$ and sample count $N_{v^*}$ for the matched clusters using
\eqal{
\begin{split}
    \label{eq:cluster_update}
    g_{v^*} & \gets \frac{N_{v^*}g_{v^*}+ \sum_{i=1}^n\embd^{\rm sim}(x_i)}{N_{v^*}+n},\\
    N_{v^*} & \gets N_{v^*} + n,
\end{split}
}
\paragraph{Pruning} We also introduce a strategy for pruning small clusters. We model the sample count of each cluster as a binomial distribution $N_v \propto B(T, p)$. We set $p=\frac{1}{V}$, assuming that each cluster is equally likely to be matched by a local class of samples. Any cluster with a sample count below the following threshold is discarded,
\eqal{
\label{eq:cluster_prune}
    N_v<\Bar{N_v} - q\sqrt{\textrm{Var}(N_v)}
}
where $\Bar{N_v}$ is the expectation of $N_v$, $\textrm{Var}(N_v)$ the variance, and $q$ a hyper-parameter controlling how aggressive the pruning is.

\begin{algorithm}[t]
   \caption{LearnLabeler\label{alg:self-label}}
\begin{algorithmic}
    \STATE \hspace{-1em}{\bfseries Input:} embedding model $\embd^{\mathrm{sim}}$, meta-training set $\Tau=\{\dtr_t, \dval_t\}_{t=1}^{T}$, number of classes in a task $k$
    
    \STATE \hspace{-1em}{\bfseries Initialization:} sample tasks from $\Tau$ to initialize clusters $G=\{g_v\}_{v=1}^{V}$,
    \vspace{0.5em}
    
    \STATE {\bfseries While} $|G|$ has not converged:
    \STATE \quad $N_v = 1$ for each $g_v\in G$
    \STATE \quad \textbf{For} $(\dtr, \dval) \in \Tau$:
            \STATE \qquad Match $\dtr \cup \dval$ to its centroids $M \,{=}\, \{g_q\}_{q=1}^K$ using \cref{eq:cluster_centroid}
            \STATE \qquad \textbf{If} $M$ has $k$ unique clusters
                \STATE \quad\qquad Update centroid $g_q$ for each $g_q\in M$ via \cref{eq:cluster_update}    
    \STATE \quad $G \leftarrow \{g_v|g_v\in G, N_v\geq\textrm{threshold in }$\cref{eq:cluster_prune}$\}$
\STATE \hspace{-1em}{\bfseries Return} $G$
\end{algorithmic}
\end{algorithm}

\cref{alg:self-label} outlines the full labeling algorithm. We first initialize a large number of clusters by setting their centroids with mean class embeddings from random classes in $\Tau$. For $V$ initial clusters, $\lceil \frac{V}{k}\rceil$ tasks are needed since each task contains $k$ classes and could initialize as many clusters. The algorithm then alternates between clustering and pruning to refine the clusters and estimate the number of clusters jointly. The algorithm terminates and returns the current clusters $G$ when the number of clusters does not change from the previous iteration. Using clusters $G$, local classes from the meta-training set can be assigned global labels with nearest neighbor matching using \cref{eq:cluster_centroid}. For tasks that fail to map to $k$ unique global labels, we simply exclude them from the pre-training process.

The key difference between \cref{alg:self-label} and the classical $K$-means algorithm~\cite{lloyd1982least} is that the proposed clustering algorithm exploits local information to guide the clustering process, while $K$-means algorithm is fully unsupervised. We will show in the experiments that enforcing local constraints is necessary for learning robust meta-representation.

\cref{alg:self-label} also indirectly highlights how global labels, if available, offer valuable information about meta-training set. In addition to revealing precisely how input samples relate to one another across tasks, global labels provide an overview of meta-training set, including the desired number of clusters and their sizes. In contrast, \cref{alg:self-label} needs to estimate both properties when only local labels are given.

~\newline\noindent\paragraph{Time Complexity}
The time complexity of training \laml{} is dominated by the computational cost of pretraining, accounting for over $70\%$ of the overall running time. From our benchmarks, the time complexity of \laml{} is comparable to those of the current state-of-the-art methods based on pre-training~\citep[e.g.,][]{ye2020few, wertheimer2021few} and significantly more efficient than methods relying on complex base learners~\citep[e.g.,][]{zhang2020deepemd}. We refer to \cref{sec:app_time} for a more detailed discussion and comparison with \cite{ye2020few,wertheimer2021few, zhang2020deepemd}.

When global labels are not available \laml{} requires performing an additional inference step to estimate them. While this stage accounts for around 20\% of the total running time, we observe in \cref{sec:exp} that it provides a significant performance boost compared to \fsl{} methods not utilizing pre-training, which are the only applicable ones in the absence of global labels.

\subsection{Meta Fine-Tuning}
\label{sec:method_finetune}
As discussed in \cref{sec:practical-considerations}, while pre-training already yields a robust meta-representation for \fsl{}, it is desirable to adapt the pre-trained representation such that the new meta-representation better matches the base learner intended for novel classes. We call this additional training {\itshape meta fine-tuning}, which is adopted by several state-of-the-art \fsl{} models~\cite{zhang2020deepemd, ye2020few, wang2020structured, li2021universal}.

For meta fine-tuning, existing works suggest that model performance depends crucially on preserving the pre-trained representation. In particular, \cite{rusu2018meta, wang2020structured, li2021universal} all keep the pre-trained representation fixed, and only learn a relatively simple transformation on top for the new base learners. Additionally, \cite{goldblum2020unraveling} showed that meta fine-tuning the entire representation model lead to worse performance compared to standard meta-learning, negating the advantages of pre-training completely.

We thus present a simple residual architecture that preserves the pre-trained embeddings and allows adaptation for the new base learner. Formally, we consider the following parameterization for a fine-tuned meta-learned embedding $\embd^{*}$,
\eqal{
    \embd^{*}(x) = \embd^{\rm pre}(x) + h(\embd^{\rm pre}(x))
}
where $\embd^{\rm pre}$ is the pre-trained representation and $h$ a learnable function (e.g. a small fully connected network). We again use \cref{eq:closed-form-solver} as the base learner and optimizes \cref{eq:meta-representation-model} directly. Our experiments show that the proposed fine-tuning process achieves results competitive with more sophisticated base learners, indicating that the pre-trained representation is the predominant contributor to good test performance.

\section{Experiments}
\label{sec:exp}
We evaluate \laml{} on various benchmark datasets and compare it with existing methods. The experiments are designed to address the following questions:
\begin{itemize}
    \item How does \laml{} compare to existing methods for generalization performance? We also introduce the more challenging \gfsl{} setting in \cref{sec:exp_setting}.
    \item How do different model components (e.g. pre-training, meta fine-tuning) contribute to generalization performance?
    \item Does \laml{} learn meaningful clusters? Can \laml{} handle conflicting task labels?
    \item How can we improve the quality of the pre-trained representation?
    \item How robust is \laml{} to hyper-parameter choices?
\end{itemize}

\subsection{Benchmark Datasets}
\label{sec:exp_data}
\textbf{Mini/tiered-ImageNet.}~\cite{vinyals2016matching, ren2018meta} has become default benchmark for \fsl{}. Both datasets are subsets of ImageNet~\cite{russakovsky2015imagenet} with \mimg{} having 60K images over 100 classes, and \timg{} having 779K images over 608 classes. Following previous works, we report performance on 1- and 5-shot settings, using 5-way classification tasks.

~\newline\paragraph{Variants of mini/tiered-ImageNet} We introduce several variants of mini/tiered-ImageNet to better understand \laml{} and more broadly the impacts of dataset configuration on pre-training. Specifically, we create mini-60 that consists of 640 classes and 60 samples per class. Mini-60 contains the same number of samples as \mimg{}, though with more classes and fewer samples per class. Mini-60 keeps the same meta-test set as \mimg{} to ensure a fair comparison of test performance of model trained on each dataset in turn. We designed mini-60 to investigate the behavior of \laml{} when encountering a dataset with a high number of base classes and low number of samples per base class. We also use mini-60 to explore how data diversity present in the training data affects the learned representation. Analogous to mini-60, we also introduce tiered-780 as a variant to \timg{}, where we take the total number of samples in \timg{} and calculate the number of samples over the full 1000 ImageNet classes, excluding those used in the meta-test set of \timg{}.

~\newline\paragraph{Meta-Dataset}~\cite{triantafillou2019meta} is a meta-learning classification benchmark combining 10 widely used datasets: ILSVRC-2012 (ImageNet)~\cite{russakovsky2015imagenet}, Omniglot~\cite{lake2015human}, Aircraft~\cite{maji13fine-grained}, CUB200~\cite{cub200}, Describable Textures (DTD)~\cite{cimpoi2014describing}, QuickDraw~\cite{quickdraw}, Fungi~\cite{fungi}, VGG Flower (Flower)~\cite{nilsback2008automated},
Traffic Signs~\cite{houben2013detection} and MSCOCO~\cite{lin2014microsoft}. We use Meta-Dataset to construct several challenging experiment scenarios, including learning a unified model for multiple domains and learning from tasks with conflicting labels.

\subsection{Experiment Settings}
\label{sec:exp_setting}
The standard \fsl{} setting~\cite{finn2017model, snell2017prototypical, ye2020few, wertheimer2021few, bateni2022beyond} assumes that a meta-distribution of tasks is available for training. This translates to meta-learners having access to an exponential number of tasks synthetically generated from the underlying dataset, a scenario unrealistic for practical applications. Recent works additionally assume access to global labels in order to leverage pre-training, in contrast with earlier methods that assume access to only local labels. We will highlight such differences when comparing different methods.

~\newline\paragraph{Generalized Few-Shot Learning (\gfsl{}) Setting} We introduce a more challenging and realistic \fsl{} setting. Specifically, we only allow access to local labels, since global ones may be inaccessible or ill-defined. In addition, we employ a \textit{no-replacement} sampling scheme when synthetically generating tasks from the underlying dataset\footnote{For instance, \mimg{} (38400 training samples) will be randomly split into around 380 tasks of 100 samples}. This sampling process limits the meta-training set to a fixed-size, a standard assumption for most machine learning problems. The fixed size also enables us to evaluates the sample efficiency of different methods. Secondly, no-replacement sampling prevents \laml{} and other meta-learners from trivially learning task relations, a key objective of meta-learning, by matching same samples across tasks. For instance, an identical sample appearing in multiple tasks would allow \laml{} to trivially cluster local classes. Lastly, the sampling process reflects any class imbalance in the underlying dataset, which might present a more challenging problem.

\subsection{Performance Comparison in Standard Setting}
\label{sec:exp_std}
We compare \laml{} to a diverse group of existing  methods on mini- and \timg{} in \cref{tab:comp}. We separate the methods into those requiring global labels and those that do not. We note that the two groups of methods are not directly comparable since global labels provides a significant advantage to meta-learners as discussed previously. The method groupings are intended to demonstrate the effect of pre-training on generalization performance.

\begin{table*}[thb]
\caption{Test accuracy of meta-learning models on \mimg{} and \timg{}.}
\begin{center}
\begin{tabular}{lcc|cc}
\toprule
  & \multicolumn{2}{c}{\mimg{}} & \multicolumn{2}{c}{\timg{}}\\
& $1$-shot & $5$-shot & $1$-shot & $5$-shot\\
\midrule
Global Labels &&&\\
\midrule
Simple CNAPS \cite{bateni2022beyond} &  $53.2 \pm -$ & $70.8 \pm -$ & $63.0 \pm -$ & $80.0 \pm - $\\
LEO \cite{rusu2018meta} & $61.7 \pm 0.7$ & $77.6 \pm 0.4$ & $66.3 \pm 0.7$ & $81.4 \pm 0.6$ \\
TASML \cite{wang2020structured} & $62.0 \pm 0.5$ & $78.2 \pm 0.5$ & $66.4 \pm 0.4$ & $82.6 \pm 0.3$\\
RFS \cite{tian2020rethinking} & $62.0 \pm 0.4$ & $79.6 \pm 0.3$  & $69.4 \pm 0.5$ & $84.4 \pm 0.3$\\
ProtoNet (with pre-train) \cite{wertheimer2021few} & $62.4 \pm 0.2$ & $80.5 \pm 0.1$ & $68.2 \pm 0.2$ &  $84.0 \pm 0.3$\\
Meta-Baseline \cite{chen2021meta} & $63.2 \pm 0.2$ & $79.3 \pm 0.2$ & $68.6 \pm 0.3$ &  $83.7 \pm 0.2$\\
FEAT \cite{ye2020few} & $\mbf{66.7 \pm 0.2}$ & $82.0 \pm 0.1$ & $70.8 \pm 0.2$ & $84.8 \pm 0.2$\\
FRN \cite{wertheimer2021few} & $66.4\pm 0.2$ & $\mbf{82.8 \pm 0.1}$ & $\mathbf{71.2 \pm 0.2}$ & $\mathbf{86.0\pm 0.2}$\\
DeepEMD \cite{zhang2020deepemd} & $65.9 \pm 0.8$ & $82.4 \pm 0.6$ & $\mbf{71.2 \pm 0.9}$ &  $\mathbf{86.0 \pm 0.6}$\\
\midrule
Local Labels &&&\\
\midrule
MAML \cite{finn2017model} & $48.7 \pm 1.8$ &  $63.1 \pm 0.9$ & $51.7 \pm 1.8$ &  $70.3 \pm 0.8$\\
ProtoNet \cite{snell2017prototypical} & $49.4 \pm 0.8$ & $68.2 \pm 0.7$ & $53.3 \pm 0.9$ & $72.7 \pm 0.7$\\
R2D2 \cite{bertinetto2018meta} & $51.9 \pm 0.2$ & $68.7 \pm 0.2$ & $65.5 \pm 0.6$ & $80.2 \pm 0.4$\\
MetaOptNet \cite{lee2019meta} & $62.6 \pm 0.6$ & $78.6 \pm 0.5$ & $66.0 \pm 0.7$ & $81.5 \pm 0.6$\\
Shot-free \cite{ravichandran2019few} & $59.0 \pm {\rm n/a}$ & $77.6 \pm {\rm n/a}$ & $63.5 \pm {\rm n/a}$ & $82.6 \pm {\rm n/a}$\\
\laml{} (pre-train only) & $64.5 \pm 0.4$ & $81.5 \pm 0.3$  & $69.5 \pm 0.5$ & $84.3 \pm 0.3$\\
\laml{} & $\mbf{65.8 \pm 0.4}$ & $\mbf{82.8 \pm 0.3}$ & $\mbf{70.5 \pm 0.5}$ & $\mbf{85.9 \pm 0.3}$\\
\bottomrule
\end{tabular}
\end{center}
\label{tab:comp}
\end{table*}

\cref{tab:comp} clearly shows that ``global-labels'' methods leveraging  pre-training generally outperform ``local-labels'' methods except \laml{}. We highlight that the re-implementation of ProtoNet in \cite{wertheimer2021few} benefits greatly from pre-training, outperforming the original by over 10\% across the two datasets. Similarly, while RFS and R2D2 both learn a fixed representation and only adapt the classifier based on each task, RFS's pre-trained representation clearly outperforms R2D2's meta-learned representation. We further note that state-of-the-art methods such as DeepEMD and FEAT are heavily reliant on pre-training and performs drastically worse in \gfsl{} setting, as we will discuss in \cref{sec:exp_general}.

In the local-labels category, \laml{} outperforms existing methods thanks to its ability to still exploit pre-training using the inferred labels. \laml{} achieves about 4\% improvement over the next best method in all settings. Across both categories, \laml{} obtains performance competitive to state-of-the-art methods such as FRN, FEAT and DeepEMD despite having no access to global labels. This indicates that \laml{} is able to infer meaningful clusters to substitute global labels and obtains performance similar to methods having access to global labels. We will provide further quantitative results on the clustering algorithm in \cref{sec:exp-clustering}. Lastly, we note that \laml{} also outperforms several methods from the ``global-label'' category, such as RFS and Meta-Baseline. We attribute \laml{}'s better performance to more robust representation via augmented pre-training and our formulation for meta fine-tuning. In particular, we explicitly preserve the pre-trained representation using residual connections, in contrast to meta fine-tuning the entire representation model as in ProtoNet and Meta-Baseline. Consistent with \cite{goldblum2020unraveling}, the results suggest that meta fine-tuning the entire representation model could negate the advantages of pre-training shown in our theoretical analysis.

\subsection{Performance Comparison in Generalized Setting}
\label{sec:exp_general}
We evaluate a representative set of few-shot learners under \gfsl{}. For this setting, we introduce two new experimental scenario using Meta-Dataset to simulate task heterogeneity.

In the first scenario, we construct the meta-training set from Aircraft, CUB and Flower, which we simply denote as "Mixed". Tasks are sampled independently from one of the three datasets. For meta-testing, we sample 1500 tasks from each dataset and report the average accuracy. The chosen datasets are intended for fine-grained classification in aircraft models, bird species and flower species respectively. Thus the meta-training tasks share the broad objective of fine-grained classification, but are sampled from three distinct domains. A key challenge of this scenario is to learn a unified model across multiple domains, without any explicit knowledge about them or the global labels.

\begin{table*}[hbt]
    \caption{Test Accuracy on Aircraft, CUB and VGG Flower (Mixed dataset). A single model is trained for each method over all tasks.}
    \label{tab:comp_meta_detail}
    \centering
    \begin{tabular}{c|cc|cc|cc|cc}
    \toprule
    & \multicolumn{2}{c}{Aircraft} & \multicolumn{2}{c}{CUB} & \multicolumn{2}{c}{VGG Flower} & \multicolumn{2}{c}{Average}\\
    & 1-shot & 5-shot & 1-shot & 5-shot & 1-shot & 5-shot & 1-shot & 5-shot\\
    \midrule
    ProtoNet~\cite{snell2017prototypical} & $35.1 \pm 0.4$ & $51.0 \pm 0.5$ & $32.7 \pm 0.4$ & $46.4 \pm 0.5$ & $56.7 \pm 0.5$ & $73.8 \pm 0.4$ & $41.7 \pm 0.7$ & $57.5 \pm 0.7$\\
    MatchNet~\cite{vinyals2016matching} & $31.4\pm0.4$ & $39.4\pm0.5$ & $42.7\pm0.5$ & $54.1\pm0.5$ & $62.5\pm0.5$ & $70.0\pm0.1$ & $45.7\pm0.1$ & $54.5\pm0.1$\\
    R2D2~\cite{bertinetto2018meta} & $67.7 \pm 0.6$ & $82.8 \pm 0.4$ & $53.8 \pm 0.5$ & $69.2 \pm 0.5$ & $65.4 \pm 0.5$ & $83.3 \pm 0.3$ &  $61.9 \pm 0.5$ & $78.6 \pm 0.4$\\
    DeepEMD~\cite{zhang2020deepemd} & $34.7 \pm 0.7$ & $47.8 \pm 1.4$ & $39.3 \pm 0.7$ & $52.1 \pm 1.4$ & $61.3 \pm 0.9$ & $74.5 \pm 1.4$ & $45.1 \pm 0.7$ & $58.1 \pm 1.2$\\
    FEAT~\cite{ye2020few} & $61.7 \pm 0.6$ & $75.8 \pm 0.5$ & $59.6 \pm 0.6$ & $73.1 \pm 0.5$ & $62.9 \pm 0.6$ & $76.0 \pm 0.4$ &  $60.9 \pm 0.7$ & $75.0 \pm 0.5$ \\
    FRN~\cite{wertheimer2021few} & $60.7 \pm 0.7$ & $77.6 \pm 0.5$ & $61.9 \pm  0.7$ & $77.7 \pm  0.5$ & $65.2 \pm  0.6$ & $81.2 \pm 0.5$ & $63.1 \pm 0.7$ & $79.7 \pm 0.5$\\
    \laml{} & $\mbf{78.2 \pm 0.5}$ & $\mbf{89.5 \pm 0.3}$ & $\mbf{73.8 \pm 0.6}$ & $\mbf{88.7 \pm 0.3}$ & $\mbf{76.6 \pm 0.4}$ & $\mbf{91.5 \pm 0.2}$ & $\mbf{76.2 \pm 0.3}$ & $\mbf{89.9 \pm 0.2}$ 
    \end{tabular}
\end{table*}

\cref{tab:comp_meta_detail} show that \laml{} outperforms all baselines under \gfsl{} setting. In particular, \laml{} achieves a large margin of 10\% improvement over the baselines, including state-of-the-art models FEAT, FRN and DeepEMD, which performed competitively against \laml{} in \cref{tab:comp}. In particular, FEAT and DeepEMD performed noticeably worse, indicating the methods' reliance on pre-trained representation and the difficulty of meta-learning from scratch with complex base learners. FRN outperforms FEAT and DeepEMD, as it is designed to also work without pre-training.

In the second scenario, we consider meta-training tasks with heterogeneous objectives, leading to conflicting task-labels and consequently ill-defined global labels. For the Aircraft dataset, each sample from the base dataset has three labels associated with it, including variant, model and manufacturer\footnote{E.g. ``Boeing 737-300'' indicates manufacturer, model, and variant} that form a hierarchy. We sample tasks based on each of the three labels and creates a meta-training set containing three different task objectives: classifying fine-grained differences between model variants, classifying different airplanes, and classifying different airplane manufacturers. To differentiate from the original dataset, we refer to this meta-training set as H-Aircraft. The training data is particularly challenging given the competing goals across different tasks: a learner is required to recognize fine-grained differences between airplane variants, while being able to identify general similarities within the same manufacturer. The meta-training data also reflects the class imbalance of underlying dataset, with samples from Boeing and Airbus over-represented.

\begin{table}[t]
\caption{Test accuracy on H-Aircraft in the generalized setting.}
\begin{center}
\begin{tabular}{lcc}
\toprule
& $1$-shot & $5$-shot\\
\midrule
ProtoNet~\cite{snell2017prototypical} & $47.8 \pm 0.5$ & $66.8 \pm 0.5$\\
MatchNet~\cite{vinyals2016matching} &  $65.6 \pm 0.2$ & $78.7 \pm 0.2$\\
R2D2 & $75.1 \pm 0.3$ & $86.4 \pm 0.2$\\
DeepEMD~\cite{zhang2020deepemd} & $51.3 \pm 0.5$ & $65.6 \pm 0.8$\\ 
FEAT~\cite{ye2020few} & $77.6 \pm 0.6$ & $87.3 \pm 0.4$\\
FRN~\cite{wertheimer2021few}  & $81.9 \pm 0.4$ & $91.0 \pm 0.2$\\
\laml{} & $\mbf{84.8 \pm 0.3}$ & $\mbf{92.9 \pm 0.2}$\\
\midrule
Oracle & $84.4 \pm 0.3$ & $93.1 \pm 0.2$\\
\bottomrule
\end{tabular}
\end{center}
\label{tab:comp_air}
\end{table}

\cref{tab:comp_air} shows that \laml{} outperforms all baselines for H-Aircraft. To approximate the oracle performance when ground truth labels were given, we optimize a supervised semantic softmax loss~\cite{ridnik2021imagenet} over the hierarchical labels. Specifically, we train the (approximate) oracle to minimize a multi-task objective combining individual cross entropy losses over the three labels. \laml{} performs competitively against the oracle, indicating the robustness of the proposed labeling algorithm in handling ill-defined labels and class imbalance.

The experimental results suggest that \laml{} performs robustly in both the standard and \gfsl{} settings. In contrast, baseline methods perform noticeably worse in the latter, due to the absence of pre-training and limited training data.\\

\paragraph{Connection to theoretical results} We comment on the empirical results so far in relation to our theoretical analysis. The empirical results strongly indicate that pre-training produces robust meta-representations for \fsl{} by exploiting contextual information from global labels. This is consistent with our observation that pre-training would achieve a smaller error than its meta-learning counterpart. On the other hand, the results also validate our hypothesis that the pre-trained representation can be further improved, since the pre-trained representation is not explicitly optimized for handling novel classes. In particular, FEAT, FRN, DeepEMD and \laml{} all outperform the pre-trained representation from \cite{tian2020rethinking} by further adapting it.

\subsection{Performance Comparison on Meta-Dataset}
We further evaluate \laml{} on the full Meta-Dataset to assess our method's generalization performance. We adopt the experiment setting of training on ImageNet only and testing on all meta-test sets~\cite{triantafillou2019meta}, to clearly evaluate out-of-distribution generalization. We note that state-of-the-art methods~\citep[e.g.][]{li2022cross,triantafillou2021learning, saikia2020optimized} on Meta-Dataset are heavily reliant on pre-training with global labels, while \laml{} only has access to a collection of \fsl{} tasks and has to infer such labels. In \cref{tab:comp_meta}, we compare \laml{} with state-of-the-art methods including fine tuning~\cite{triantafillou2019meta}, ALFA+fo-Proto-MAML~\cite{triantafillou2019meta}, BOHB~\cite{saikia2020optimized}, FLUTE~\cite{triantafillou2021learning} and TSA~\cite{li2022cross}.

The results show that \laml{} is able to effectively infer meaningful global labels and achieve robust generalization to novel datasets, achieving an average accuracy of 68.5\%. Despite not being given global labels for pre-training, \laml{} only trails behind TSA while outperforming other methods. In addition, we note that the task-specific tuning adopted by TSA is orthogonal -- but compatible -- to \laml{}: by combining \laml{} with TSA (see \cref{para:tsa} for details) we are able to further improve our generalization performance, outperforming the original TSA approach on 10 out of the 13 meta-test sets (\cref{tab:comp_meta} last column). These results further demonstrate the robustness of \laml{} in learning robust representations over a large number of \fsl{} tasks, and the efficacy of task-specific fine-tuning in improving generalization of novel tasks.

\begin{table*}[thb]
\caption{Test accuracy on Meta-Dataset, training on ImageNet only, using ResNet-18 for all models.\label{tab:comp_meta}}
\begin{center}
\begin{adjustbox}{center}
\begin{tabular}{lcccccc|cc}
\toprule
\multirow{2}{*}{Test Dataset} & \multirow{2}{*}{Finetune \cite{triantafillou2019meta}} & \multirow{2}{2cm}{fo-Proto-MAML \cite{triantafillou2019meta}} & \multirow{2}{*}{BOHB \cite{saikia2020optimized}} & \multirow{2}{*}{FLUTE \cite{triantafillou2021learning}} & \multirow{2}{2cm}{Meta-Baseline \cite{chen2021meta}} & \multirow{2}{*}{TSA \cite{li2022cross}} & \multirow{2}{*}{MeLa} & \multirow{2}{2cm}{MeLa+TSA}\\
&&&&&&&&\\
\midrule
ImageNet & $45.8 \pm 1.1$ & $52.8 \pm 1.1$ & $51.9 \pm 1.1$ & $46.9 \pm 1.1$ & $59.2 \pm -$ & $57.4 \pm 1.0$ & $59.3 \pm 1.1$ & $\mathbf{61.3 \pm} 1.1$\\
Omniglot & $60.9 \pm 1.6$ & $61.9 \pm 1.5$ & $67.6 \pm 1.2$ & $61.6 \pm 1.4$ & $69.1 \pm -$ & $74.2 \pm 1.2$ & $66.2 \pm 1.4$ & $\mathbf{74.8 \pm} 1.4$\\
Aircraft & $68.7 \pm 1.3$ & $63.4 \pm 1.1$ & $54.1 \pm 0.9$ & $48.5 \pm 1.0$ & $54.1 \pm -$ & $66.1 \pm 1.0$ & $67.9 \pm 1.0$ & $\mathbf{80.7 \pm} 1.1$\\
Birds & $57.3 \pm 1.3$ & $69.8 \pm 1.1$ & $70.7 \pm 0.9$ & $47.9 \pm 1.0$ & $77.3 \pm -$ & $73.9 \pm 0.9$ & $78.7 \pm 0.8$ & $\mathbf{81.6 \pm} 0.9$\\
Textures & $69.0 \pm 0.9$ & $70.8 \pm 0.9$ & $68.3 \pm 0.8$ & $63.8 \pm 0.8$ & $76.0 \pm -$ & $76.2 \pm 0.7$ & $77.0 \pm 0.8$ & $\mathbf{78.9 \pm} 0.8$\\
QuickDraw & $42.6 \pm 1.2$ & $59.2 \pm 1.2$ & $50.3 \pm 1.0$ & $57.5 \pm 1.0$ & $57.3 \pm -$ & $64.6 \pm 0.9$ & $64.3 \pm 1.0$ & $\mathbf{71.5 \pm} 1.0$\\
Fungi & $38.2 \pm 1.0$ & $41.5 \pm 1.2$ & $41.4 \pm 1.1$ & $31.8 \pm 1.0$ & $45.4 \pm -$ & $46.8 \pm 1.1$ & $\mathbf{47.3 \pm} 1.2$ & $47.3 \pm 1.2$\\
VGG Flower & $85.5 \pm 0.7$ & $86.0 \pm 0.8$ & $87.3 \pm 0.6$ & $80.1 \pm 0.9$ & $89.6 \pm -$ & $91.3 \pm 0.5$ & $89.9 \pm 0.7$ & $\mathbf{93.5 \pm} 0.8$\\
Traffic Sign & $66.8 \pm 1.3$ & $60.8 \pm 1.3$ & $51.8 \pm 1.0$ & $46.5 \pm 1.1$ & $66.2 \pm -$ & $82.5 \pm 0.9$ & $65.4 \pm 1.1$ & $\mathbf{86.9 \pm} 1.0$\\
MSCOCO & $34.9 \pm 1.0$ & $48.1 \pm 1.1$ & $48.0 \pm 1.0$ & $41.4 \pm 1.0$ & $\mathbf{55.7 \pm} -$ & $55.2 \pm 1.0$ & $54.2 \pm 1.1$ & $54.4 \pm 1.1$\\
MNIST & - & - & - & $80.8 \pm 0.8$ & - & $94.0 \pm 0.5$ & $88.1 \pm 0.6$ & $\mathbf{94.6 \pm} 0.6$\\
CIFAR-10 & - & - & - & $65.4 \pm 0.8$ & - & $\mathbf{79.4 \pm} 0.8$ & $70.4 \pm 0.8$ & $76.9 \pm 0.8$\\
CIFAR-100 & - & - & - & $52.7 \pm 1.1$ & - & $\mathbf{70.5 \pm} 0.9$ & $62.0 \pm 1.0$ & $69.3 \pm 0.9$\\
\midrule
Average & 57.0 & 61.4 & 59.1 & 55.8 & 65.0 & 71.7 & 68.5 & \textbf{74.7}\\
Average Rank & 6.0 & 5.2 & 5.6 & 6.5 & 4.0 & 2.4 & 3.0 & \textbf{1.3}\\
\bottomrule
\end{tabular}
\end{adjustbox}
\end{center}
\end{table*}

\subsection{Ablations on Pre-training}
\label{sec:exp_pretrain}
Given the significance of pre-training on final performance, we investigate how the rotation data augmentation and data configuration impact the performance of the pre-trained representation. We focus on the effects of dataset sizes and the number of classes present in the dataset.\\

\paragraph{Rotation-Augmented Pre-training} In \cref{sec:method_rot}, we proposed to increase both the size and the number of classes in a dataset via input rotation. By rotating the input images by the multiples of $90^\circ$, we quadruple both the size and the number of classes in a dataset. In \cref{tab:comp_rotation}, we compare the performance of standard pre-training against the rotation-augmented one, for multiple datasets. We use the inferred labels from \laml{} for pre-training.

\begin{table}[t]
\caption{Test accuracy comparison between pre-trained representations: standard vs. rotation-augmented}
\begin{center}
\scriptsize
\begin{tabular}{lcccc}
\toprule
& \multicolumn{2}{c}{$1$-shot} & \multicolumn{2}{c}{$5$-shot}\\
& standard & rotation & standard & rotation\\
\midrule
\mimg{} & $62.0 \pm 0.4$ & $64.5 \pm 0.4$ & $79.6 \pm 0.3$ & $81.5 \pm 0.3$\\
mini-60 & $63.9 \pm 0.7$ & $67.7 \pm 0.5$ & $81.5 \pm 0.5$ & $84.1 \pm 0.5$\\
\timg{} & $69.1 \pm 0.5$ & $69.5 \pm 0.6$ & $83.9 \pm 0.3$ &  $84.3 \pm 0.4$\\
tiered-780 & $78.0 \pm 0.6$ & $78.2 \pm 0.6$ & $89.9 \pm 0.4$ & $90.1 \pm 0.4$\\
H-Aircraft & $79.2 \pm 0.5$ & $84.8 \pm 0.3$ & $89.4 \pm 0.3$ & $92.9 \pm 0.2$\\
\bottomrule
\end{tabular}
\end{center}
\label{tab:comp_rotation}
\end{table}

The results suggest that rotation-augmented pre-training consistently improves the quality of the learned representation. It achieves over 3\% improvements in both \mimg{} and H-aircraft, while obtains about 0.5\% in \timg{}. It is clear that rotation augmentation works the best with smaller datasets with fewer classes. As the dataset increases in size and diversity, the additional augmentation has less impact on the learned representation.

~\newline\paragraph{Effects of Class Count} We further evaluate the effects of increasing number of classes in a dataset while maintaining the dataset size fixed. For this, we compare the performance of \mimg{} and \timg{} with their respective variants mini-60 and tiered-780.

\cref{tab:comp_rotation} suggests that given a fixed size dataset, having more classes improves the quality of the learned representation compared to having more samples per class. We hypothesize that classifying more classes lead to more discriminative and robust features, while standard $\ell_2$ regularization applied during pre-training prevents overfitting despite having fewer samples per class.

Overall, the experiments suggest that pre-training is a highly scalable process where increasing either data diversity or dataset size will lead to more robust representation for \fsl{}. In particular, the number of classes in the dataset appears to play a more significant role than the dataset size.

\subsection{Ablations on The Clustering Algorithm}\label{sec:exp-clustering}
The crucial component of \laml{} is \cref{alg:self-label}, which infers a notion of global labels and allows pre-training to be exploited in \gfsl{} setting. We perform several ablation studies to better understand \cref{alg:self-label}.\\

\paragraph{The Effects of No-replacement Sampling} We study the effects of no-replacement sampling, since it affects both the quality of the similarity measure through $\embd^{\rm sim}$ and the number of tasks available for inferring global clusters. The results are shown in \cref{tab:embd_0}.

\begin{table*}[tb]
    \centering
    \caption{The effects of no-replacement sampling on the clustering algorithm}
    \begin{tabular}{l|cc|cc|cc}
    \toprule
    Dataset & \multicolumn{2}{c}{\mimg{}} & \multicolumn{2}{c}{\timg{}} & \multicolumn{2}{c}{mini-60}\\
    Replacement & Yes & No & Yes & No & Yes & No\\
    \midrule
        Tasks Clustered (\%) & 100 & 98.6 & 99.9 & 89.5 & 98.7 & 97.8 \\
        Clustering Acc (\%) & 100 & 99.5 & 96.4 & 96.4 & 72.8 & 70.1 \\
        1-shot Acc (\%) & $65.8 \pm 0.4$ & $65.8 \pm 0.5$ & $70.5 \pm 0.5$ & $70.5 \pm 0.5$ & $68.4 \pm 0.7$ & $68.4 \pm 0.5$ \\
        5-shot Acc (\%) & $82.8 \pm 0.3$ & $82.8 \pm 0.4$ & $85.9 \pm 0.3$ & $85.9 \pm 0.3$ & $84.0 \pm 0.5$ & $84.0 \pm 0.5$ \\
    \end{tabular}
    \label{tab:embd_0}
\end{table*}

In \cref{tab:embd_0}, clustering accuracy is computed by assigning the most frequent ground truth label in each cluster as the desired target. Percentage of tasks clustered refers to the tasks that map to $k$ unique clusters by \cref{alg:self-label}. The clustered tasks satisfy both constraints imposed by local labels and are used for pre-training.

For both sampling processes, \laml{} achieves comparable performances across all three datasets. This indicates the robustness of \cref{alg:self-label} in inferring suitable labels for pre-training, even when task samples do not repeat across tasks. This shows that \cref{alg:self-label} is not trivially matching identical samples across task, but relying on $\embd^{\rm sim}$ for estimating sample similarity. We note that mini-60 is particularly challenging under no-replacement sampling, with only 384 tasks in the meta-training set over 640 ground truth classes.\\

\paragraph{Effects of Pruning Threshold}
In \cref{alg:self-label}, the pruning threshold is controlled by the hyper-parameter $q$. We investigate how different $q$ values affect the number of clusters estimated by the algorithm and the corresponding test accuracy. \cref{tab:q_sens} suggest that \laml{} is robust to a wide range of $q$ and obtains representations similar to that produced by the ground truth labels. While it is possible to replace $q$ with directly guessing the number of clusters in \cref{alg:self-label}, we note that tuning for $q$ is more convenient since appropriate $q$ values appear to empirically concentrate within a much narrower range, compared to the possible numbers of global clusters present in a dataset.\\

\begin{table*}[thb]
    \caption{Test accuracy (pre-train only) and cluster count for various pruning thresholds, 5-shot setting}
    \centering
    \begin{tabular}{ccc|ccc|ccc}
    \toprule
    \multicolumn{3}{c|}{\mimg{} (64 classes)} & \multicolumn{3}{c|}{\timg{} (351 classes)} & \multicolumn{3}{c}{mini-60 (640 classes)}\\
    \multicolumn{3}{c|}{Oracle Pre-train: 81.5\%} & \multicolumn{3}{c|}{Oracle Pre-train: 84.5\%} & \multicolumn{3}{c}{Oracle Pre-train: 85.2\%}\\
    \midrule
    $q$ & No. Clusters & \laml{} & $q$ & No. Clusters & \laml{} & $q$ & No. Clusters & \laml{}\\
     3 & 64 & $81.5 \pm 0.4$ & 3.5 & 351 & $84.3 \pm 0.4$ & 4.5 & 463 & $84.0 \pm 0.4$\\
     4 & 75 & $81.1 \pm 0.4$ & 4.5 & 373 & $84.1 \pm 0.3$ & 5.5 & 462 & $83.8 \pm 0.4$\\
     5 & 93 & $80.9 \pm 0.4$ & 5.5 & 444 & $84.0 \pm 0.5$ & 6.5 & 472 & $84.0 \pm 0.4$\\
    \end{tabular}
     \label{tab:q_sens}
\end{table*}

\paragraph{Inferred Labels vs. Oracle Labels}
From \cref{tab:embd_0,tab:q_sens}, we observe that it may be unnecessary to fully recover the oracle labels (when they exists). For mini-60, \laml{} inferred 463 clusters over 640 classes, which implies mixing of the oracle classes. However, the inferred labels still perform competitively against the oracle labels, suggesting the robustness of the proposed method. The results also suggest that we may improve the recovery of the oracle labels by sampling more tasks from the meta-distribution.\\

\paragraph{The Importance of Local Constraints}
\cref{alg:self-label} enforces consistent assignment of task samples given their local labels. To understand the importance of enforcing these constraints, we consider an ablation study where \cref{alg:self-label} is replaced with the standard $K$-means algorithm. The latter is fully unsupervised and ignores any local constraints. We initialize the $K$-means algorithm with 64 clusters for \mimg{} and 351 clusters for \timg{}, matching the true class counts in respective datasets.

\begin{table*}[thb]
    \footnotesize
    \centering
    \caption{Test Accuracy (Pre-train only) using \cref{alg:self-label} vs. $K$-mean Clustering}
    \begin{tabular}{l|ccc|ccc}
    \toprule
    & \multicolumn{3}{c}{\mimg{}} & \multicolumn{3}{c}{\timg{}}\\
    Cluster Alg. & Cluster Acc & 1-shot & 5-shot & Cluster Acc & 1-shot & 5-shot\\ 
    \midrule
    \cref{alg:self-label} (\laml{}) & 100 & $64.5\pm 0.4$ & $81.5\pm 0.3$ & 96.4 & $69.5\pm 0.5$ & $84.3\pm 0.3$\\
    $K$-mean & 84.9 & $60.7\pm 0.5$ & $76.9\pm 0.3$ & 28.2 & $64.8\pm 0.6$ & $78.8\pm 0.5$ \\
    \end{tabular}
    \label{tab:comp_kmean}
\end{table*}

\cref{tab:comp_kmean} indicates that enforcing local constraints is critical for generalization performance during meta-testing. In particular, test accuracy drops by over 5\% for \timg{}, when the $K$-means algorithm ignores local task constraints. Among the two constraints, we note that \cref{eq:cluster_centroid} appears to be the more important one since nearly all tasks automatically match $K$ unique clusters in our experiments (see tasks clustered in \cref{tab:embd_0}).\\

\paragraph{Domain Inference for multi-domain tasks}
In addition to inferring global labels, We may further augment \cref{alg:self-label} to infer the different domains present in a meta-training set, if we assume that all samples within a task belongs to a single domain. Given the assumption, two global clusters are connected if they both contain samples from the same task. This is illustrated in \cref{fig:domain_inf}. Consequently, inferred clusters form an undirected graph with multiple connected components, with each representing a domain. We apply the above algorithm to the multi-domain Mixed Dataset consisting of Aircraft, CUB and Flower.

\cref{fig:meta_domain} visualizes the inferred domains on the multi-domain scenario. For each inferred cluster, we project its centroid onto a 2-dimensional point using UMAP~\cite{mcinnes2018umap}. Each connected component is assigned a different color. Despite some mis-clustering within each domain, we note that \cref{alg:self-label} clearly separates the three domains present in the meta-training set and recovers them perfectly.

Domain inference is important for multi-domain scenario as it enables domain-specific pre-training. Recent works~\citep[e.g.][]{liu2020universal, li2021universal, dvornik2020selecting} on Meta-Dataset have shown that combining domain-specific representation into a universal representation is empirically more advantageous than training on all domains together. Lastly, we remark that multi-domain meta-learning is also crucial for obtaining robust representation suitable for wider range of novel tasks, including cross-domain transfer.

\begin{figure*}[t]
    \centering
    \subfloat[\textbf{Domain Inference via Connected Components}]{
    \includegraphics[width=0.5\textwidth]{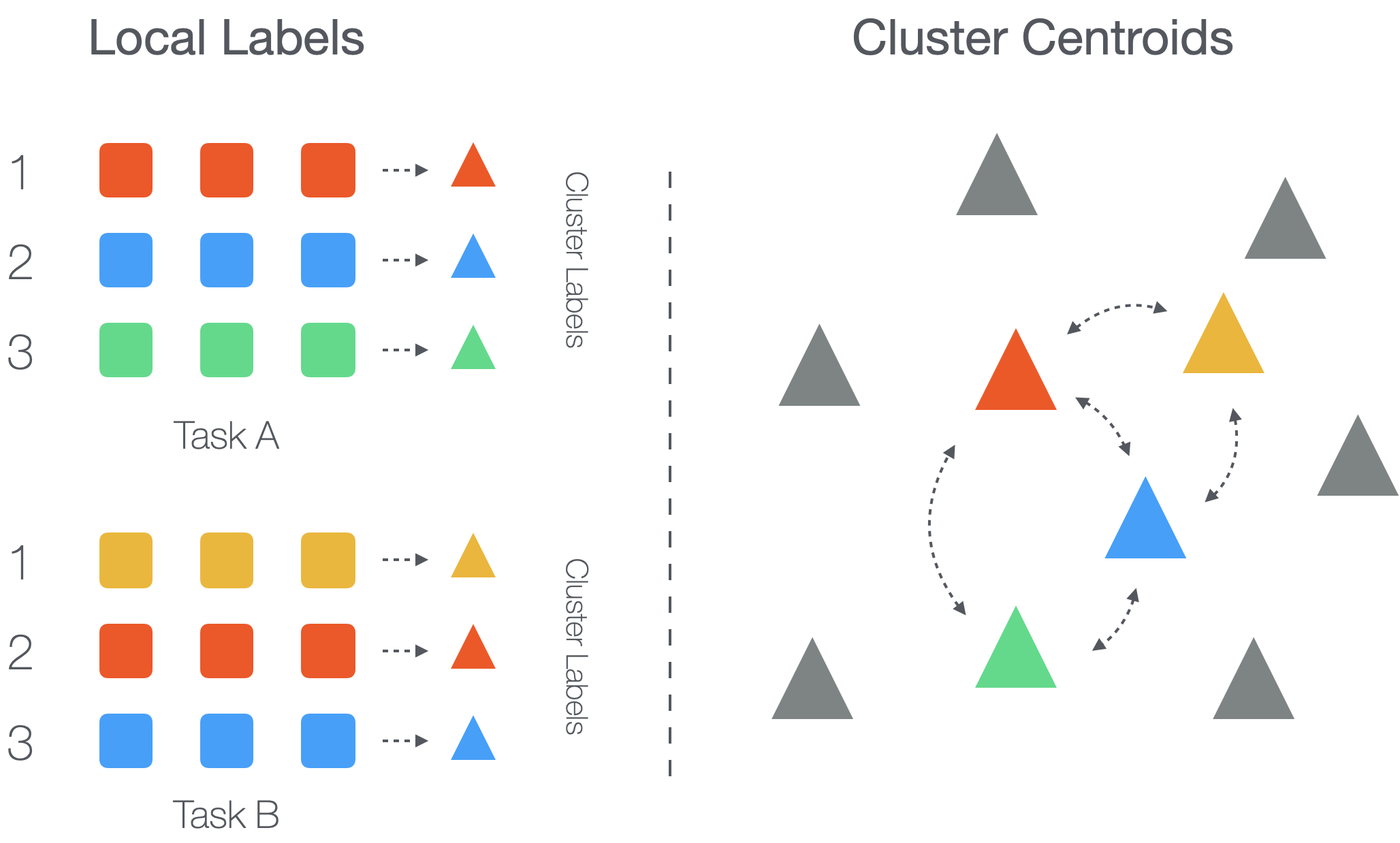}\label{fig:domain_inf}}
    \hfil
    \subfloat[\textbf{UMAP visualization of different domains in multi-domain dataset}]{
    \includegraphics[width=0.35\textwidth]{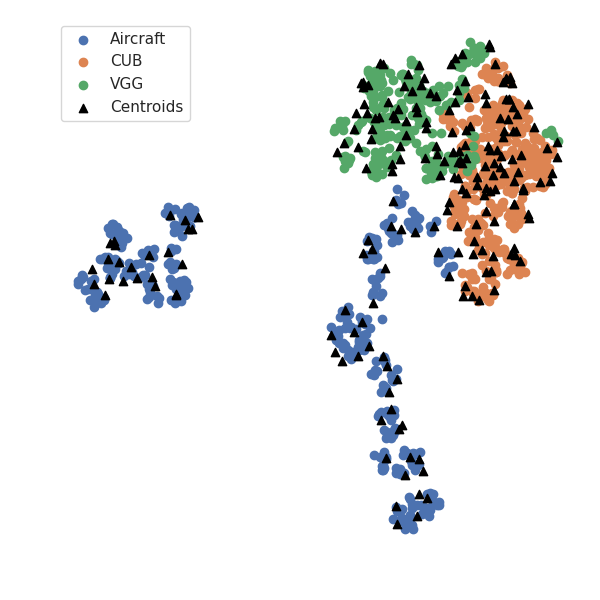}\label{fig:meta_domain}}
    \caption{{\bfseries (a)} The coloured clusters (red, green, blue and yellow) are connected since they both contains samples from the same task. Domains can be inferred by computing the connected components of the inferred clusters. {\bfseries (b)} UMAP visualization of the three inferred domains from the 5-shot Mixed dataset containing Aircraft, CUB and VGG. Circles are the means (using the pretrained features) of the instances in each task averaged per local class while triangles are the learned centroids, all vectors are embedded using UMAP. The three domains are recovered perfectly.}
\end{figure*}

\section{Conclusion}
In this work we focused on the role played by pre-training in meta-learning applications, with particular attention to few-shot learning problems. Our analysis was motivated by the recent popularity of pre-training as a key stage in most state-of-the-art \fsl{} pipelines. We first investigated the benefits of pre-training from a theoretical perspective. We showed that in some settings this strategy enjoys significantly better sample complexity than pure meta-learning approaches, hence offering a justification for its empirical performance and wide adoption in practice. 

We then proceeded to observe that pre-training requires access to global labels of the classes underlying the \fsl{} problem. This might not always be possible, due to phenomena like heterogeneous labeling (i.e. multiple labelers having different labeling strategies) or contextual restrictions like privacy constraints. We proposed Meta-Label Learning (\laml{}) as a strategy to address this concern. We compared \laml{} with state-of-the-art methods on a number of tasks including well-established standard benchmarks as well as new datasets we designed to capture the above limitations on task labels. We observed that \laml{} is always comparable or better than previous approaches and very robust to lack of global labels or the presence of conflicting labels. 

More broadly, our work provides a solid foundation for understanding existing \fsl{} methods, in particular the vital contribution of pre-training towards generalization performance. We also demonstrated that pre-training scales well with the size of datasets and data diversity, which in turn leads to more robust few-shot models. Future research may focus on further theoretical understanding of pre-training and better pre-training processes.

\appendices

\ifCLASSOPTIONcompsoc
  \section*{Acknowledgments}
\else
  \section*{Acknowledgment}
\fi

Ruohan Wang acknowledges funding from Career Development Fund (grant C210812045) from A{*}STAR Singapore. John Isak Texas Falk acknowledges funding from the computer science department at UCL. Massimiliano Pontil acknowledges financial support from PNRR MUR project PE0000013-FAIR and the European Union (Projects 951847 and 101070617). Carlo Ciliberto acknowledges founding from the Royal Society (grant SPREM RGS$\backslash$ R1$\backslash$201149) and Amazon Research Award
(ARA).

\ifCLASSOPTIONcaptionsoff
  \newpage
\fi

\bibliographystyle{abbrvnat}

\bibliography{biblio}

\begin{IEEEbiography}[{\includegraphics[width=1in,height=1.25in,clip,keepaspectratio]{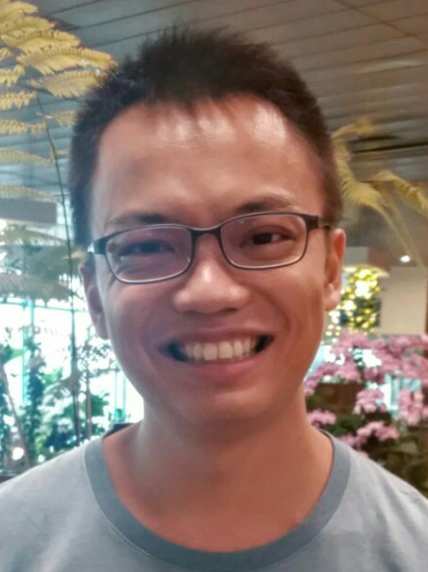}}]{Ruohan Wang}

is a research scientist at Institute for Infocomm Research, A*STAR Singapore. Previously, he was a postdoctoral researcher at UCL’s Intelligent Systems Group, supervised by Prof. Massimiliano Pontil. He completed his Ph.D. in machine learning at Imperial College under the supervision of Prof. Yiannis Demiris and Prof. Carlo Ciliberto, funded by Singapore National Science Scholarship. He has a broad interests in topics of machine learning, including representation learning, meta-learning, and imitation learning. His research goal is to design robust ML systems that could efficiently leverage past experiences and existing knowledge for future learning.
\end{IEEEbiography}
\vskip -2.4\baselineskip plus -1fil
\begin{IEEEbiography}[{\includegraphics[width=1in,height=1.25in,clip,keepaspectratio]{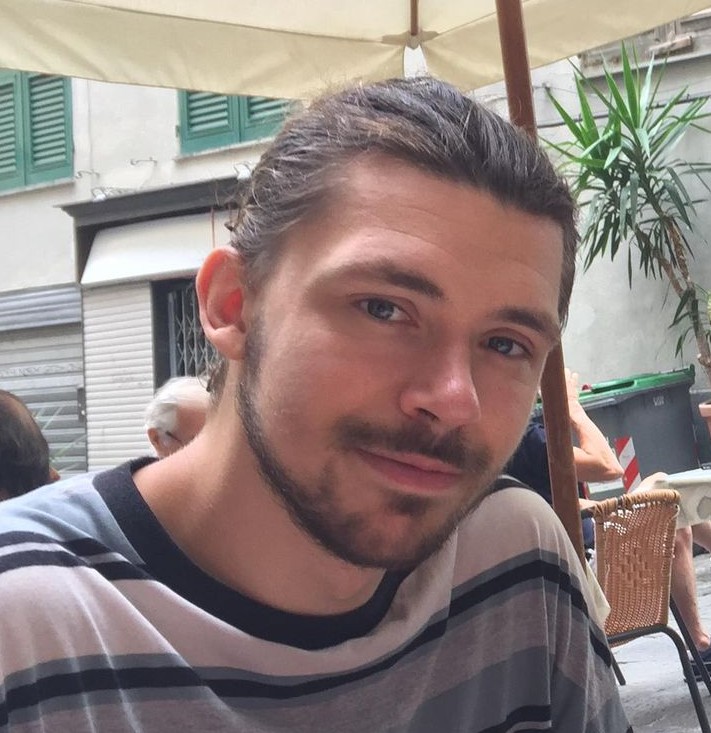}}]{John Isak Texas Falk}
is a Ph.D student at the Dept. of Computer Science of UCL where he is supervised by Prof. Massimiliano Pontil and Prof. Carlo Ciliberto. He is also a Research Fellow at Italian Institute of Technology CSML under the supervision of Prof. Massimiliano Pontil. His Ph.D is focused on few-shot and meta-learning in the context of kernel learning and statistical learning theory in the context of meta-learning, although he has an interest in representation learning, fairness and bandits. The goal of his research is how to design robust meta-learning systems that work well under real-world conditions and come with guarantees.
\end{IEEEbiography}
\vskip -2.4\baselineskip plus -1fil
\begin{IEEEbiography}[{\includegraphics[width=1in,height=1.25in,clip,keepaspectratio]{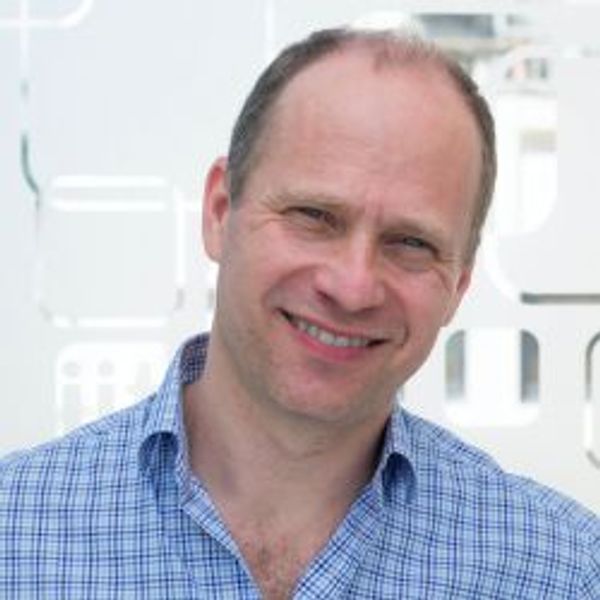}}]{Massimiliano Pontil} is Senior Researcher at the Italian Institute of Technology (IIT), where he leads the Computational Statistics and Machine Learning group, and co-director of the ELLIS Unit Genoa, a joint effort of IIT and the University of Genoa. He is also part-time professor at University College London and member of the UCL Centre for Artificial Intelligence. He obtained his PhD in Physics from the University of Genoa in 1999. He has made significant contributions to machine learning, particularly in the areas of kernel methods, multitask and transfer learning, sparsity regularization and statistical learning theory. He has published about 150 papers at international journals and conferences, is regularly on the programme committee of the main machine learning conferences, and has been on the editorial board of the Machine Learning Journal, Statistics and Computing, and JMLR. He has held visiting positions at a number of universities and research institutes, including the Massachusetts Institute of Technology, the Isaac Newton Institute for Mathematical Sciences in Cambridge, the City University of Hong Kong, the University Carlos III de Madrid, ENSAE Institute Polytechnique Paris, and Ecole Polytechnique.
\end{IEEEbiography}
\vskip -2.4\baselineskip plus -1fil
\begin{IEEEbiography}[{\includegraphics[width=1in,height=1.25in,clip,keepaspectratio]{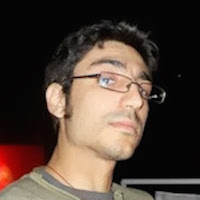}}]{Carlo Ciliberto}
is Associate Professor with the Centre for Artificial Intelligence at University College London, He is member of the ELLIS society and of the ELLIS Unit based at UCL. He obtained his bachelor and master degrees in Mathematics at the Università Roma Tre (Magna Cum Laude) and a PhD in machine learning applied to robotics and computer vision at the Istituto Italiano di Tecnologia. He has been Postdoctoral Researcher at the Massachusetts Institute of Technology with the Center for Brain Minds and Machines and became Lecturer (Assistant Professor) at Imperial College London before joining UCL, where he now carries out his main research activity. Carlo's research interests focus on foundational aspects of machine learning within the framework of statistical learning theory. He is particularly interested in the role of “structure” (being it in the form of prior knowledge or structural constraints) in reducing the sample complexity of learning algorithms with the goal of making them more sustainable both computationally and financially. He investigated these questions within the settings of structured prediction, multi-task and meta-learning, with applications to computer vision, robotics and recommendation systems.
\end{IEEEbiography}

\onecolumn

\section{Proof for Theorem 1}

We first restate the notations and assumptions necessary to prove \cref{thm:gls-and-pretraining}. From \cref{asm:task-distribution}, we define the joint distribution $\pi_\mu(x,y)$ as the probability of observing an input-output pair $(x,y)$ when first sampling a task $\rho$ from $\mu$ and then sampling $(x,y)$ according to the ``query'' distribution $\pi_\rho$ in \cref{eq:sampling-query-set}. In other words, $\pi_\mu(x, y) = \pi(x|y)\unif_\rho(y)\mu(\rho)$ is the marginal distribution of $(x, y)$ with respect to $\rho$.

\noindent We observe that for any $g:\X\times\Y\to\R$, 
\eqal{\label{eq:marginalizing-out-rho}
    \EE_{(x,y)\sim\pi_\mu}~g(x,y) = \EE_{\rho\sim\mu}\EE_{(x,y)\sim\pi_\rho}~g(x,y).
}
Given \cref{asm:task-distribution}, $S$ and $Q$ are sampled independently by the task $\rho$. In particular the marginal $\rho_Q$ of $\rho$ with respect to $Q$ corresponds to $\pi_\rho$. Similarly, we denote $\rho_S$ the distribution over support sets obtained by marginalizing out the query set. We report one remark following the assumption above.

\begin{remark}\label{rem:metarisk-and-inner-risk}
For any task $\rho$ and any algorithm $D\mapsto f_D$ returning functions $f_S:\X\to\R^k$, we have
\eqals{
    \EE_{(S,Q)\sim\rho}~\Lagr(f_S,Q) & = \EE_{S\sim\rho_S}\EE_{Q\sim\pi_\rho^m}~\Lagr(f_S,Q)\\
    & = \EE_{S\sim\rho_S}\EE_{Q\sim\pi_\rho^m} \frac{1}{m} \sum_{(x,y)\in Q} \lagr(f_S(x),y)\\
    & = \EE_{S\sim\rho_S}\EE_{(x,y)\sim\pi_\rho} ~\lagr(f_S(x),y)\\
    & = \EE_{S\sim\rho_S}~\Lagr(f_S,\pi_\rho),
}
where for any $f:\X\to\R^k$ we have denoted by 
\eqals{
    \Lagr(f,\pi_\rho) = \EE_{(x,y)\sim\pi_\rho} ~\lagr(f_S(x),y),
}
the expected risk of a function $f:\X\to\R^k$ with respect to the loss $\lagr$ and the distribution $\pi_\rho$. 
\end{remark}
\noindent Using the above remark, we shows how the \gls{} risk for the meta-distribution $\mu$ is related to multi-class classification risk for the corresponding multi-class distribution $\pi_\mu$.
\glsandpretraining*
\begin{proof}
We start by studying the \gls{} risk for a pair $(W,\theta)$ of meta-parameters. By expanding the \gls{} objective explicitly with \cref{rem:metarisk-and-inner-risk}, we have
\eqals{
    \E_{\gls}(W,\theta,\mu) & = \EE_{\rho\sim\mu}~\Lce(\gls(W,\theta,\rho),\pi_\rho)\\
    & = \EE_{\rho\sim\mu}~\EE_{(x,y)\sim\pi_\rho}~\lce(\gls(W,\theta,\rho)(x),y)\\
    & = \EE_{\rho\sim\mu}~\EE_{(x,y)\sim\pi_\rho}~\lce(W[\rho_Y]\psi_\theta(x),y).
}
Since $\rho_Y$ is a subset of $\{1,\dots,C\}$, by definition of cross-entropy we have
\eqals{
    \lce(W[\rho_Y]\psi_\theta(x),y) & = - \log \frac{\exp(W[y]\psi_\theta(x))}{\sum_{y'\in\rho_Y}\exp(W[y']\psi_\theta(x))}\\
    & \leq - \log \frac{\exp(W[y]\psi_\theta(x))}{\sum_{y'\in\{1,\dots,C\}}\exp(W[y']\psi_\theta(x))}\\
    & = \lce(W\psi_\theta(x),y).
}
We can now apply \cref{eq:marginalizing-out-rho} where we take $g(x,y) = \lce(W\psi_\theta(x),y)$. Then,
\eqals{
    \EE_{\rho\sim\mu}~\EE_{(x,y)\sim\pi_\rho} ~\lce(W\psi_\theta(x),y) = \EE_{(x,y)\sim\pi_\mu} ~\lce(W\psi_\theta(x),y) = \Lce(W,\theta,\pi_\mu),
}
which concludes the proof for \cref{eq:gls-and-pretraining}.

\vspace{1em}
\noindent Now, if the global classes are linearly separable, we have
\eqal{
\label{eq:pretrain_sep}
     \min_{W,\metapar}~ \Lagr(W\embd,\pi_\mu) = 0.}
Since $\ell_{ce}(\cdot)$ is non-negative, combining \cref{eq:pretrain_sep} and \cref{eq:gls-and-pretraining} yields
\eqal{
    0 \le \min_{W,\metapar}~\E_{\rm GLS}(W,\metapar) \le  \min_{W,\metapar}~ \Lagr(W\embd,\pi_\mu)= 0}
from which \eqref{eq:gls-and-pretraining-min} follows.
\end{proof}

\section{Comparison between the Lipschitz Constants of $\E_\gls$ and $\Lagr$}\label{sec:app-lipschitz}

In this section, we compare the Lipschitz constants associated to the meta-GLS risk and the pre-training objective function discussed in \cref{sec:generalization}, showing that $L_{\rm pre}$ is comparable or smaller than $L_{\gls}$.  

We recall that the Lipschitz constant of an objective function of the form $\mathbb{E}_{\xi\sim\rho}\ell(\cdot,\xi)$ is $\sup_{\xi} L_\xi$, where $L_\xi$ is the Lipschitz constant of $\ell(\cdot,\xi)$ for any $\xi$ in the support of the probability $\rho$, denoted as $\textrm{supp}\rho$. Furthermore, we recall that if $\ell(\cdot,\xi)$ is smooth on its domain $\Omega$ then
the Lipschitz constant of $\mathbb{E}_{\xi\sim\rho}\ell(\cdot,\xi)$ can be characterized as
\eqals{
    \sup_{\substack{\omega\in\Omega \\ \xi\in\textrm{supp}\rho}} \nor{\nabla_\omega\ell(\omega,\xi)}, 
}

In the following we will assume $\Omega = B_{\lambda}^{2,1}(0) \times \Theta$, where $B_\lambda(0)$ is a ball of radius $\lambda$ with respect to the $2,1$--norm centred at $0$, namely that $\nor{W}_{2,1}\leq\lambda$. The $2,1$--norm is defined as $\nor{W}_{2,1} = \sum_{z\leq C} \nor{W_z}_2$ and helps to simplify the following analysis. $\Theta$ is the space of parameters for $g_\theta$ and we assume the representation model to be normalized, namely $\nor{g_\theta(x)}=1$ for all $\theta\in\Theta$ and $x\in\X$. The normalized representational model could be easily generalized to any bounded representation, namely $\sup_{\theta,x}\nor{g_\theta(x)}<+\infty$.

We now compare the Lipschitz constants of $\E_\gls$ in \cref{eq:expected-gls} and the pre-training risk \cref{eq:risk-multiclass}. The key difference between the two objective functions is that meta-GLS applies the cross-entropy loss over subsets of $\{1,\dots,C\}$ classes at each time (the classes appearing in a given task), while global multi-class classification applies it to the entire set of $C$ classes. Therefore, to highlight the dependency on the number of classes in the following we denote $\lce^k$ the cross entropy evaluated among $k$ classes for $k$ an integer. Then, under the notation introduced in this section and in \cref{sec:generalization}, we need to compare 
\eqals{
L_\gls = \sup_{\substack{(W,\theta)\in\Omega \\ (S,Q)\in\textrm{supp}\tilde\mu}}~ \nor{\nabla_{W,\theta}~
\mathcal{L}
(W[S_Y]g_\theta(\cdot),Q)} = \sup_{\substack{(W,\theta)\in\Omega \\ (S,Q)\in\textrm{supp}\tilde\mu \\ (x,y)\in Q} }~\nor{\nabla_{W,\theta}~\lce^k(W[S_Y]g_\theta(x),y)},
}
with the Lipschitz constant of the global multi-class classifier loss, corresponding to
\eqals{
L_{\rm pre} = \sup_{\substack{(W,\theta)\in\Omega \\ (x,y)\in\textrm{supp}\pi_\mu}}~ \nor{\nabla_{W,\theta}~\lce^C(W g_\theta(x),y)},
}
where $\tilde\mu$ denotes the probability of sampling a pair $(S,Q)$ by first sampling $\rho\sim\mu$ and then $(S,Q)\sim\rho$.

We first observe that if $\textrm{supp}\pi_\mu = \textrm{supp}\tilde\mu$ then $L_\gls = L_{\rm pre}$. This happens {\itshape if the few-shot learning distribution can sample datasets containing at least one example per class}. We note however that the two quantities are in general very close to each other when the FSL datasets contain at most $k < C $ classes each, as shown in the result below. 

\begin{theorem}
With the notation and assumptions introduced above
\eqals{
    L_{\rm pre} \leq L_\gls +
    O(e^{-\lambda})
}
\end{theorem}

\begin{proof}
We proceed by studying the norm of the gradients of $\lce^k$ with respect to $W$ and $\theta$ separately.

~\newline\noindent\paragraph{Gradient with respect to $W$}
Given a pair of $(x,y)$ and parameters $\theta$ and $W\in\R^{k \times m}$, the derivative with respect to $W_y$ (the $y$-th column of $W$) is
\eqals{
    \nabla_{W_y} ~ \lce^k(W g_\theta(x),y) = \left(\frac{e^{W_y^\top g_\theta(x)}}{\sum_{z\leq k}e^{W_z^\top g_\theta(x)}} - 1\right)~g_\theta(x).
}
The gradient with respect to a column $W_z$ for $z\neq y$ is
\eqals{
    \nabla_{W_z} ~ \lce^k(W g_\theta(x),y) = \frac{e^{W_z^\top g_\theta(x)}}{\sum_{z\leq k}e^{W_z^\top g_\theta(x)}} ~g_\theta(x).
}
We conclude that the norm of the gradient with respect to the full $W$ is
\eqals{
\nor{\nabla_W ~ \lce^k(W g_\theta(x),y)} & = \nor{g(x)}\sqrt{\left(1 - \frac{e^{W_y^\top g_\theta(x)}}{\sum_{z\leq k}e^{W_z^\top g_\theta(x)}} \right)^2 + \sum_{z\neq y}\left(\frac{e^{W_z^\top g_\theta(x)}}{\sum_{z\leq k}e^{W_z^\top g_\theta(x)}}\right)^2}   \leq \sqrt{2}
}
where we have used the fact that representations are normalized $\nor{g_\theta(x)}=1$, and, for any $W$, the two terms under the square root are  bounded by $1$. This is clear for the first term, while the second one is the squared $\ell_2$ norm of a vector in the simplex (minus the component associated to $W_y$).
Hence
\eqals{
    \sup_{(W,\theta)\in\Omega}~\nor{\nabla_{W_z} ~ \lce^k(W g_\theta(x),y)} \leq \sqrt{2}
}

Now, we provide a lower bound to the gradient norm. Let $\bar z\neq y$ and evaluate the gradient at a $W$ such that $W_z = 0$ for any $z \neq \bar z$ ($y$ included) and $W_{\bar z} = \lambda g_\theta(x)$ for $\lambda>0$, we have
\eqals{
\nor{\nabla_W ~ \lce^k(W g_\theta(x),y)} = \frac{1}{e^{\lambda} + k-1} \sqrt{(e^{\lambda} + k-2)^2 + e^{2\lambda} + k-2} = \frac{1}{1 + \frac{k-1}{e^{\lambda}} } \sqrt{\left(1 + \frac{k-2}{e^{2\lambda}}\right)^2 
 + 1 + \frac{k-2}{e^{\lambda}}}
}
Note that as $\lambda\to+\infty$, the r.h.s. of the above equation converges to $\sqrt 2$ as fast as $O(e^{-\lambda})$. Consequently, 
for any $\theta$
\eqals{
    \left|\sup_{(W,\theta)\in\Omega} ~\nor{\nabla_W ~ \lce^k(W g_\theta(x),y)} - \sqrt 2 \right| \leq O(e^{-\lambda})
}
The gradient norm converges to $\sqrt 2$ exponentially fast with respect to the upper bound on norm of $W$. In particular we have that for any number of classes $k$, 
\eqals{
    \sup_{(W,\theta)\in\Omega}~\nor{\nabla_{W_z} ~ \lce^k(W g_\theta(x),y)} \geq \sqrt{2} - O(e^{-\lambda}).
}
Since $\sqrt{2}$ is an upper bound, we have that for any $k\leq C$ and any $\theta\in\Theta$
\eqals{
    \sup_{(W,\theta)\in\Omega}~\nor{\nabla_{W_z} ~ \lce^C(W g_\theta(x),y)} - \sup_{(W,\theta)\in\Omega}~\nor{\nabla_{W_z} ~ \lce^k(W g_\theta(x),y)} \leq \sqrt{2} - (\sqrt{2} - O(e^{-\lambda})) = O(e^{-\lambda}),
}
which implies that the component of the norm of the gradient of $\lce^C$ is at most larger than the same component but for $\lce^k$ of a quantity that decreases exponentially fast with respect to $\lambda$.

~\newline\noindent\paragraph{Gradient with respect to $\theta$}
We now consider the gradient with respect to $\theta$. Given an input-output pair $(x,y)$ and a linear classifier $W$, we have
\eqals{
    \nabla_\theta \lce^k(Wg_\theta(x),y) = - \nabla_\theta g_\theta(x)^\top \left[\sum_{z\neq y} \frac{e^{W_z^\top g_\theta(x)}}{\sum_{z\leq k} e^{W_z^\top g_\theta(x)}} (W_y-W_z)\right].
}
To estimate the maximum of the norm of the gradient above with respect to the parameters $W$, we will show below that the maximum is achieved by choosing a class $\bar z\neq y$ such that $W_z=0$ for all $z \neq \bar z$ and $W_{\bar z}$ is the only non-zero column.

Let $W_y = \lambda_y ~v$ for some vector $v\in\R^m$ of norm one $\nor{v}=1$ and $\lambda_y>0$. We start by observing that the maximum length vector for the gradient above is obtained by summing vectors that are all aligned, and therefore choosing $W_z = - \lambda_z v$ for any $z\neq y$, with an appropriate scaling $\lambda_z \in \R$ is optimal. The gradient above becomes
\eqals{
    \nabla_\theta \lce^k(W g_\theta(x),y) = - \nabla_\theta g_\theta(x)^\top v ~\left[\sum_{z\neq y}\frac{(\lambda_y+\lambda_z) e^{-\lambda_z v^\top g_\theta(x)}}{\sum_{z\leq k} e^{-\lambda_z v^\top g_\theta(x)}}\right],
}
and its sup with respect to $W$ is
\eqals{
    \max_{\nor{W}_{2,1}\leq\lambda}~\nor{\nabla_\theta \lce^k(Wg_\theta(x),y)} = \max_{\nor{v}=1} \left[\nor{\nabla_\theta g_\theta(x)^\top v} ~ \left|\max_{\sum_{z\leq k}|\lambda_z|\leq\lambda}~\sum_{z\neq y}\frac{(\lambda_y+\lambda_z) e^{\lambda_z v^\top g_\theta(x)}}{\sum_{z\leq k} e^{-\lambda_z v^\top g_\theta(x)}}\right|\right],
}
where we have used the fact that the constraint $\nor{W}_{2,1}\leq\lambda$ on the linear parameters $W$ corresponds to the constraint $\sum_{z\leq k} |\lambda_z| \leq \lambda$.

We first note that to achieve the maximum, $\lambda_z\geq0$ for all $z\neq y$, otherwise the term $\lambda_y+\lambda_z$ would be smaller than $\lambda_y + |\lambda_z|$ (note that the term in the exponential is not affected by the sign of $\lambda_z$, since we can choose $v$ such that $v^\top g_\theta(x)$ has either sign and $\nor{\nabla_\theta g_\theta(x)^\top v} = \nor{-\nabla_\theta g_\theta(x)^\top v}$. This implies that we do not need the absolute value on the term
\eqals{
\max_{\sum_{z\leq k}|\lambda_z|\leq\lambda}~\sum_{z\neq y}\frac{(\lambda_y+\lambda_z) e^{\lambda_z v^\top g_\theta(x)}}{\sum_{z\leq k} e^{-\lambda_z v^\top g_\theta(x)}}.
}
Additionally, we note that the maximum above is achieved for any set of $(\lambda_z)_{z\leq k}$ such that $\lambda_{\bar z}> 0$ for some $\bar z \neq y$ and $\lambda_z = 0$ for all $z\neq \bar z$. This follows by noting that 
\eqals{
\max_{\sum_{z\leq k}\lambda_z\leq\lambda}~\sum_{z\neq y}\frac{(\lambda_y+\lambda_z) e^{-\lambda_z v^\top g_\theta(x)}}{\sum_{z\leq k} e^{-\lambda_z v^\top g_\theta(x)}} & \leq \max_{\sum_{z\leq k}\lambda_z\leq\lambda}~\sum_{z\neq y}\frac{e^{-\lambda_z v^\top g_\theta(x)}}{\sum_{z\leq k} e^{-\lambda_z v^\top g_\theta(x)}}~ (\lambda_y+ \max_{z\neq y} \lambda_z).
}
Now, both $\max_{z\neq y} \lambda_z$ and $
\sum_{z\neq y}\frac{e^{-\lambda_z v^\top g_\theta(x)}}{\sum_{z\leq k} e^{-\lambda_z v^\top g_\theta(x)}}$ are maximized by choosing all $\lambda_z$ to be equal to zero except for one. In particular the inequality above becomes an equality if $\lambda_{\bar z} = \lambda$ for some $\bar z \neq y$ and $\lambda_z = 0$ for any $z\neq \bar z$ (including $y$). Plugging this estimation of the maximum in the gradient norm derived above, we have
\eqals{
    \max_{\nor{W}_{2,1}\leq\lambda}~\nor{\nabla_\theta \lce^k(Wg_\theta(x),y)} = \max_{\nor{v}=1} \nor{\nabla g_\theta(x)^\top v} \frac{\lambda e^{-\lambda v^\top g_\theta(x)}}{e^{-\lambda v^\top g_\theta(x)} + k-1}
}
We note now that, for any $\lambda$ and $v\in\R^m$, the quantity that we are maximizing in the equation above is strictly decreasing with respect to $k$. This means in particular that
\eqals{\label{eq:lip-gradient-theta}
    \max_{(W,\theta)\in\Omega}~\nor{\nabla_\theta \lce^k(Wg_\theta(x),y)} \leq \max_{(W,\theta)\in\Omega}~\nor{\nabla_\theta \lce^{C}(Wg_\theta(x),y)}
}

~\newline\noindent\paragraph{Putting the two together}
Let us consider the norm squared of $\lce^C$. We have
\eqals{
    L_{\rm pre} = \sup_{(W,\theta)\in\Omega}~\nor{\nabla_{W,\theta} ~ \lce^C(W g_\theta(x),y)}^2 & = \sup_{(W,\theta)\in\Omega}~ \nor{\nabla_{W} ~ \lce^C(W g_\theta(x),y)}^2 + \nor{\nabla_{\theta} ~ \lce^C(W g_\theta(x),y)}^2 \\
    & \leq \sup_{(W,\theta)\in\Omega}~ \nor{\nabla_{W} ~ \lce^k(W g_\theta(x),y)}^2 + O(e^{-2\lambda}) + \nor{\nabla_{\theta} ~ \lce^k(W g_\theta(x),y)}^2 \\
    & = \sup_{(W,\theta)\in\Omega}~ \nor{\nabla_{W,\theta} ~ \lce^k(W g_\theta(x),y)}^2 + O(e^{-2\lambda})= L_{\gls}^2 + O(e^{-2\lambda}),
}
where for the inequality we have used the fact that, for both gradients (with respect to $W$ or $\theta$), the maximum with respect to $W$, for a fixed $\theta$ is achieved by selecting a $W$ that is zero along all directions but one. Since $\sqrt{a+b} \leq \sqrt{a} + \sqrt{b}$ for any $a,b\geq0$, we conclude that $L_{\rm pre} \leq L_\gls + O(e^{-\lambda}),$
as required.
\end{proof}

~\newline\noindent The theorem above shows that the Lipschitz constant for the (global) multi-class classification risk is at most larger than that for GLS by an amount that decreases exponentially fast as $\lambda$ increases. From regularization theory we know that $\lambda$ grows proportionally to the number of samples observed, hence we can expect that for all practical purposes $L_{\rm pre}\leq L_{\gls}$.

\section{Experimental Details}
In this section we specify the datasets and hyperparameter settings for the experiments in the main body. The training follow the same steps regardless of dataset, we specify the available hyperparameters for each step below. We denote the steps (in sequence as they appear) of \cref{alg:meta_learner} by \textsc{RepLearn}, \textsc{LearnLabeler}, \textsc{PreTrain}, \textsc{MetaLearn}, and evaluation by \textsc{Eval}. The code repository with implementation details can be found at \url{https://github.com/isakfalk/mela}. In the code base we use PyTorch \cite{paszke2017automatic}, scikit-learn \cite{pedregosa2011scikit}, numpy \cite{harris2020array}, matplotlib \cite{hunter2007matplotlib} and umap \cite{mcinnes2018umap}.

\subsection{Datasets}

\paragraph{\textit{mini}/\textit{tiered}ImageNet}
We use the standard \mimg{} and \timg{} dataset.\\

\paragraph{\textit{mini}-60 and \textit{tiered}-780}
We provide code for generating the datasets of \textit{mini}-60 and \textit{tiered}-780 using the ILSVRC2012 ImageNet dataset and global labels provided (\url{https://www.image-net.org/challenges/LSVRC/2012/index.php}). Code and instructions for creating these datasets can be found in \url{https://github.com/isakfalk/mela}. \\

\paragraph{Meta-Dataset}
We use the official Meta-Dataset \cite{triantafillou2019meta} \href{https://github.com/google-research/meta-data}{implementation} and pre-process them according to the instructions provided.\\

\subsection{Model Details}
\label{sec:app_model}

\paragraph{Combining \laml{} with Task-specific Adapters}
\label{para:tsa}
In \laml{}, the prediction model $f=\ridge\cdot\embd$ for a task $D$ consists of the embedding model $\embd$ learned via our proposed approach, and a linear classifier $\ridge$ (e.g., ridge regression or logistic regression) derived from $D$. This model $f$ can further fine-tuned using task-specific adapters (TSA)~\cite{li2022cross}, which further adapts the embedding model $\embd$ to better match the data distribution of $D$. Concretely, TSA inserts residual layers to each level of $\embd$. While keeping the original parameters fixed, TSA then fine-tunes the new layers and classifier $\ridge_t$ via a cross-entropy loss on $D$. The two strategies are therefore fully compatible, with \laml{} recovering the labels and then running TSA as per the original paper \cite{li2022cross} with some minor modifications (e.g. the inclusion of Dropblocks as described in the {\bfseries Backbone Architecture} paragraph in this section.

\paragraph{Optimization}
\label{para:optimization}
We use two instances of optimization algorithms, relying on the optimization library of PyTorch
\begin{itemize}
    \item \textsc{SGD}: SGD with an initial learning rate of 0.05, weight decay factor of \(0.0005\) and momentum of \(0.9\)
    \item \textsc{AdamW}: AdamW \cite{loshchilov2017decoupled} with learning rate of 0.0001, weight decay factor of \(10^{-6}\).
\end{itemize}
For each optimization algorithm we use the torch multi-step learning rate scheduler which anneals the learning rate of the optimization algorithm by \(\gamma = 0.1\) at selected epochs in \texttt{lr\_decay\_epochs}.
\begin{itemize}
    \item \textsc{MultiStepLR}: Learning rate annealing scheduler, which multiplies the learning rate by \(\gamma\) at the beginning of epochs in the list \texttt{lr\_decay\_epochs}.
\end{itemize}

\paragraph{Augmentation}
\label{para:augmentation}
We use two instances of augmentation (and one option of no augmentation)
\begin{itemize}
    \item \textsc{DataAug}: Data augmentation where we use a pipeline of
    \begin{enumerate}
        \item Random cropping using a shape of \(84 \times 84\) with padding of \(8\)
        \item Color jittering with the PyTorch arguments of (\texttt{brightness=0.4, contrast=0.4, saturation=0.4})
        \item Randomly flip the image horizontally.
        \item Normalization: (channel-wise) using ImageNet sample channel mean and standard deviation for ImageNet type datasets, min-max scaling to \([-1, 1]\) for Mixed and H-aircraft datasets
    \end{enumerate}
    \item \textsc{RotateAug}: Rotation-class augmentation as laid out in \cref{sec:method_rot} together with \textsc{DataAug}
    \item \textsc{None}: No augmentation.
\end{itemize}

\paragraph{Backbone Architecture}
\label{para:architecture}
We use ResNets \cite{he2016deep} for the backbone throughout the experiments. We observe that the inclusion of Dropblocks~\cite{ghiasi2018dropblock} improves model generalization.
\begin{itemize}
    \item \textsc{ResNet12}: ResNet with block sequence \([1, 1, 1, 1]\), using adaptive average pooling, drop-blocks for the final 2 ResNet layers and a drop rate of 0.1, with output dimension being \(640\)
    \item \textsc{ResNet18}: ResNet with block sequence \([1, 1, 2, 2]\), using adaptive average pooling, drop-blocks for the final 2 ResNet layers and a drop rate of 0.1, with output dimension being \(640\)
\end{itemize}

\begin{table}[H]
\caption{Hyperparameters for all datasets.}
\label{tab:hyperparams}
\centering
    \begin{tabular}{l|c|c|c|c|c|c|} 
        \toprule
         & \mimg{} & \timg{} & mini-60 & tiered-780 & Mixed & H-aircraft\\
        \midrule
        \multicolumn{7}{l}{\textsc{RepLearn}}\\
        \midrule
        Architecture & \textsc{ResNet12} & \textsc{ResNet12} & \textsc{ResNet12} & \textsc{ResNet12} & \textsc{ResNet12} & \textsc{ResNet12}\\
        Augmentation & \textsc{None} & \textsc{None} & \textsc{None} & \textsc{None} & \textsc{None} & \textsc{None}\\
        Epochs & 30 & 40 & 30 & 40 & 50 & 50 \\
        \texttt{lr\_decay\_epochs} & [20, 25] & [28, 36] & [20, 25] & [28, 36] & [30, 42] & [30, 42]\\
        $T_{\mathrm{tasks}}$ (if not \textsc{GFSL}) & 2800 & 2800 & 2800 & 2800 & 2800 & 2800\\
        \midrule
        \multicolumn{7}{l}{\textsc{LearnLabeler}}\\
        \midrule
        Architecture & \textsc{ResNet12} & \textsc{ResNet12} & \textsc{ResNet12} & \textsc{ResNet12} & \textsc{ResNet12} & \textsc{ResNet12}\\
        Augmentation & \textsc{None} & \textsc{None} & \textsc{None} & \textsc{None} & \textsc{None} & \textsc{None}\\
        q & 3 & 3.5 & 4.5 & 4.5 & 3.5 & 3.5\\
        $K_{\mathrm{init}}$ & 600 & 1200 & 700 & 1100 & 400 & 400\\
        \midrule
        \multicolumn{7}{l}{\textsc{Pretrain}}\\
        \midrule
        Architecture & \textsc{ResNet12} & \textsc{ResNet18} & \textsc{ResNet12} & \textsc{ResNet12} & \textsc{ResNet12} & \textsc{ResNet12}\\
        Augmentation & \textsc{RotateAug} & \textsc{RotateAug} & \textsc{RotateAug} & \textsc{RotateAug} & \textsc{RotateAug} & \textsc{RotateAug}\\
        \midrule
        \multicolumn{7}{l}{\textsc{MetaLearn}}\\
        \midrule
        Epochs & 3 & 3 & 3 & 3 & 3 & 3\\
        \bottomrule
    \end{tabular}
\end{table}

\paragraph{Residual Adapters for Meta Fine-Tuning}
In \cref{sec:method_finetune}, we introduced a residual adapter for meta fine-tuning. The learnable network $h$ is a three-layer MLP with 
\textsc{ResNet\{12/18\}+ResFC}: MLP with residual connection and layer-normalization applied to the output. Both the input and output dimensions are the same as the feature representation from either \textsc{ResNet12} or \textsc{ResNet18} backbone.\\

\paragraph{Learning the Similarity Measure (\textsc{RepLearn})}
\label{para:replearn}
For training embedding \(\embd^{\mathrm{sim}}\), when given a task \(D = (S \cup Q)\) we one-hot encode the outputs and scale them using \(f(y) = 2y - 1\). We get the classifier \(w(\embd^{\mathrm{sim}}(S))\) using \cref{eq:closed-form-solver} on the embedded support set \(\embd^{\mathrm{sim}}(S)\) (we add a column of ones to the embeddings for a bias term) with a regularization strength of \(\lambda_{\mathrm{MetaLS}} = 0.001\). As an inner loss we use \(\ell = \ell_{\mathrm{FS}}\) where \(\ell_{\mathrm{FS}}\) is the few-shot loss using mean-squared error inner loss \cref{eq:closed-form-solver}. We train for a fixed number of epochs, where each epoch is a full sweep over the meta-train set in the \textsc{GFSL} setting or some predefined number of tasks \(T_{\mathrm{tasks}}\) in the standard setting. Number of tasks in each batch is set to 1. We use meta-validation set for early stopping and model selection.\\

\paragraph{Global Label Inference (\textsc{LearnLabeler})}
\label{para:learn-labeler}
Given a trained backbone \(\embd^{\mathrm{sim}}\) we use the clustering algorithm of \cref{para:clustering} with the hyperparameters \(q\) and \(K_{\mathrm{init}}\) where \(q\) is the pruning aggression parameter and \(K_{\mathrm{init}}\) is the initial number of centroids on the same meta-train few-shot dataset on which \(\embd^{\mathrm{sim}}\) was trained, which gives rise to centroids \(G\) and in extension the inferred global labels for standard multi-class classification.\\

\paragraph{Pre-Training via Multi-class Classification (\textsc{PreTrain})}
\label{para:pre-train}
Given a flat supervised dataset with inferred labels, we train a backbone \(\embd^{\mathrm{\rm pre}}\) using the cross-entropy loss as in \cref{eq:std_classify} with \textsc{SGD} and one of the data augmentation strategies outlined in \cref{para:augmentation}. We train for a set number of epochs, where each epoch is a full sweep over the \textit{flattened} meta-train set in the \textsc{GFSL} setting or some predefined number of tasks \(T_{\mathrm{tasks}}\) (inferred labels, flattened) in the standard setting. For the Oracles we use the same procedure on the full flat supervised dataset with the ground-truth labels. For the H-Aircraft dataset with multiple ground-truth labels, we use the semantic softmax \cite{ridnik2021imagenet} with augmentation \textsc{DataAug}. We use meta-validation set for early stopping and model selection.\\

\paragraph{Meta Fine-tuning (\textsc{MetaLearn})}
\label{para:meat-learn}
We adapt the pre-trained representation \(\embd^{\mathrm{\rm pre}}\) towards $\embd^{*}$ by combining the original backbone with the residual adapter. We use augmentation \textsc{DataAug} without rotation and optimize \cref{eq:meta-representation-model} using \textsc{AdamW}. We train for a set number of epochs, where each epoch is a sweep over the meta-train sets. Number of tasks per batch i set to 1. We use meta-validation set for early stopping and model selection. We perform 5 trials for each experiment setup for reporting performance standard deviation.\\

\paragraph{\textsc{Model Evaluation}}
\label{para:eval}
During model evaluation (either using meta-validation or meta-testing set), we obtain the embedding of each task sample using the trained feature extractor. Each sample embedding is normalized to unit-length before being passed to the classifier $\ridge(\cdot)$. We follow \cite{tian2020rethinking} and uses logistic regression from Scikit-learn as the classifier. We set the regularization strength as 1.0 for pre-trained feature extractor $\embd^{\mathrm{\rm pre}}$ and 0.001 for the fine-tuned feature extractor $\embd^*$.\\

\paragraph{Hyper-parameters for \laml{}} The hyper-parameters and the values used in the experiments are listed in \cref{tab:hyperparams}.

\paragraph{Hyper-parameters for previous methods} For standard settings including \mimg{}, \timg{} and Meta-Dataset, we use the recommended hyper-parameters from each method's original paper and directly report their original performance where appropriate.

For \gfsl{}, we performed hyper-parameter search for the evaluated methods such that they could perform well when trained from random initialization. We list the key hyper-parameter values for the evaluated methods below, with the rest set to their default values from their official implementation:

\begin{itemize}
    \item \textbf{ProtoNet}: $\textrm{lr}=0.05, \textrm{epochs}=30, \textrm{lr\_decay\_epochs}=[20, 25], \textrm{temperature}=0.1$.
    \item \textbf{MatchNet}: $\textrm{lr}=0.05, \textrm{epochs}=30, \textrm{lr\_decay\_epochs}=[20, 25]$.
    \item \textbf{R2D2}: $\textrm{lr}=0.05, \textrm{epochs}=30, \textrm{lr\_decay\_epochs}=[20, 25], \textrm{weight\_decay}=0.0005. \lambda=0.001$.
    \item \textbf{DeepEMD}: $\textrm{lr}=0.01, \textrm{lr}_{sfc}=0.1, \textrm{epochs}=50, \textrm{weight\_decay}=10^{-6}, \textrm{num\_patch}=9$.
    \item \textbf{FRN}: $\textrm{lr}=0.1, \textrm{epochs}=20, \textrm{weight\_decay}=0.0005. \gamma=0.1$.
    \item \textbf{FEAT}: $\textrm{lr}=0.05, \textrm{lr}_{mul}=1, \textrm{epochs}=30, \textrm{lr\_decay\_epochs}=[20, 25], \textrm{temperature}_1=64, \textrm{temperature}_2=32$.
\end{itemize}

\subsection{Computational Cost}
\label{sec:app_time}

We compare the training wall-time of \laml{} against DeepEMD, FEAT and FRN, three state-of-the-art methods for \fsl{}. We divide the running time into three stages, including global label inference (only applicable to \laml{}), pre-training and meta fine-tuning. \cref{tab:mela-timings} shows that \laml{}'s time complexity is comparable to FEAT and FRN while outperforming DeepEMD. We note that DeepEMD's base learner solves a computationally expensive objective function, resulting in slow training. For \laml{}, the running time is dominated by pre-training due to the proposed rotational augmentation from \cref{sec:method_rot}, while the meta fine-tuning stage is very lightweight. In particular, the pre-trained embedding is fixed and only a small adapter is optimized. In contrast, the other three algorithms (DeepEMD, FEAT and FRN) spends significantly more time during meta fine-tuning due to more complex model architecture and the need to optimize the entire model.

\begin{table}[H]
\caption{Comparing the wall-clock time for different \fsl{} methods}
\label{tab:mela-timings}
\centering
    \begin{tabular}{l|cccc} 
        \toprule
        & \multicolumn{4}{c}{Time (s)}\\
        & \textsc{Global Label Inference} & \textsc{Pre-Train} & \textsc{Meta Fine-tuning} & Total\\
        \midrule
        Mixed Dataset & & & & \\
        \midrule
        FRN & n.a & 4718 & 3200 & 7918\\
        FEAT & n.a & 1258 & 4800 & 6058\\
        DeepEMD & n.a & 1258 & 28950 & 30208\\
        MeLa & 1436 & 5031 & 72 & 6539\\
        \midrule
        \mimg{} & & & &\\
        \midrule
        FRN & n.a & 5726 & 5910 & 11636 \\
        FEAT & n.a & 1527  & 6900 & 8427\\
        DeepEMD & n.a & 1527  & 27810 & 29337\\
        MeLa & 1724 & 6108 & 144 & 7976 \\
        \midrule
        \timg{} & & & & \\
        \midrule
        FRN & n.a & 29863 & 41100 & 70963 \\
        FEAT & n.a & 12798 & 60300 & 73098\\
        DeepEMD & n.a & 12798 & 422450 & 435248\\
        MeLa & 20040 & 51195 & 144 & 71379 \\
        \bottomrule
    \end{tabular}
\end{table}

\end{document}